\documentclass{tlp}
\usepackage{url}
\usepackage{tkz-graph}
\usepackage{pgfplots}
\usepackage{bbm} 
\usepackage{amssymb}
\usepackage{hyperref}
\usepackage{algorithm}
\usepackage{algpseudocode}

\usepackage{subcaption}

\newcommand{\revisione}[1]{\textcolor{black}{#1}}

\definecolor{red_plot}{HTML}{F30522}
\definecolor{first_plot}{HTML}{F30522}

\definecolor{green_plot}{HTML}{20D525}
\definecolor{second_plot}{HTML}{20D525}

\definecolor{dark_blue_plot}{HTML}{2166AC}
\definecolor{third_plot}{HTML}{2166AC}

\definecolor{yellow_plot}{HTML}{DED712}
\definecolor{fourth_plot}{HTML}{DED712}

\definecolor{orange_plot}{HTML}{FFD0B4}
\definecolor{fifth_plot}{HTML}{FFD0B4}

\definecolor{light_blue_plot}{HTML}{D1E5F0}
\definecolor{sixth_plot}{HTML}{D1E5F0}

\newcommand{\resizeGraphFactor}{0.85}

\newcommand{\impl}{\ {:\!-}\  } 
\DeclareMathOperator{\lowerprob}{\underline{P}}
\DeclareMathOperator{\upperprob}{\overline{P}}
\newtheorem{example}{Example}
\newtheorem{definition}{Definition}
\newtheorem{theorem}{Theorem}

\usepackage{listings}
\begin{document}

\lefttitle{Cambridge Author}

\jnlPage{xx}{xx}
\jnlDoiYr{2024}
\doival{10.1017/xxxxx}

\title[Hybrid Probabilistic Answer Set Programming]{Probabilistic Answer Set Programming with Discrete and Continuous Random Variables}

\begin{authgrp}
\author{\gn{Damiano} \sn{Azzolini} }
\affiliation{Department of Environmental and Prevention Sciences -- University of Ferrara, Ferrara, Italy \\
\email{damiano.azzolini@unife.it}}
\author{\gn{Fabrizio} \sn{Riguzzi} }
\affiliation{Department of Mathematics and Computer Science -- University of Ferrara, Ferrara, Italy\\
\email{fabrizio.riguzzi@unife.it}}
\end{authgrp}

\history{\sub{xx xx xxxx;} \rev{xx xx xxxx;} \acc{xx xx xxxx}}

\maketitle

\begin{abstract}
  Probabilistic Answer Set Programming under the credal semantics (PASP) extends Answer Set Programming with probabilistic facts that represent uncertain information.
  The probabilistic facts are discrete with Bernoulli distributions.
  However, several real-world scenarios require a combination of both discrete and continuous random variables.
  In this paper, we extend the PASP framework to support continuous random variables and propose Hybrid Probabilistic Answer Set Programming (HPASP).
  Moreover, we discuss, implement, and assess the performance of two exact algorithms based on projected answer set enumeration and knowledge compilation and two approximate algorithms based on sampling.
  Empirical results, also in line with known theoretical results, show that exact inference is feasible only for small instances, but knowledge compilation has a huge positive impact on the performance.
  Sampling allows handling larger instances, but sometimes requires an increasing amount of memory.
  Under consideration in Theory and Practice of Logic Programming (TPLP).
\end{abstract}


\maketitle

\section{Introduction}
\label{sec:introduction}
Almost 30 years ago~\citep{DBLP:conf/iclp/Sato95}, Probabilistic Logic Programming (PLP) was proposed for managing uncertain data in Logic Programming.
Most of the frameworks, such as PRISM~\citep{DBLP:conf/iclp/Sato95} and ProbLog~\citep{DBLP:conf/ijcai/RaedtKT07}, attach a discrete distribution, typically Bernoulli or Categorical, to facts in the program and compute the success probability of a query, i.e., a conjunction of ground atoms.
Recently, several extensions have been proposed for dealing with continuous distributions as well~\citep{AzzRigLam21-AIJ-IJ,DBLP:conf/ilp/GutmannJR10,DBLP:journals/tplp/GutmannTKBR11}, that greatly expand the possible application scenarios.

Answer Set Programming~\citep{brewka2011asp} is a powerful formalism for representing complex combinatorial domains.
Most of the research in this field focuses on deterministic programs.
There are three notable exceptions: LPMLN~\citep{DBLP:conf/kr/LeeW16}, P-log~\citep{DBLP:journals/tplp/BaralGR09}, and Probabilistic Answer Set Programming under the credal semantics~\citep{DBLP:conf/ilp/CozmanM16}.
The first assigns weights to rules while the last two attach probability values to atoms.
However, all three allow discrete distributions only.

In this paper, we extend Probabilistic Answer Set Programming under the credal semantics to \textit{Hybrid} Probabilistic Answer Set Programming (HPASP), by adding the possibility of expressing continuous distributions.
Our approach is based on a translation of the hybrid probabilistic answer set program into a regular probabilistic answer set program (with discrete random variables only) via a process that we call \textit{discretization}.
In this way we can adapt and leverage already existing tools that perform inference in probabilistic answer set programs under the credal semantics with discrete variables only.
We implemented the discretization process on top of two exact solvers, based on projected answer set enumeration and knowledge compilation, respectively, and tested them on a set of five benchmarks with different configurations.
Furthermore, we also implemented and tested two approximate algorithms based on sampling on the original program (with both discrete and continuous facts) and on the discretized program (with only discrete facts).
Our experiments show that knowledge compilation has a huge impact in reducing the execution time and thus in increasing the scalability of the inference task.
However, larger instances require approximate algorithms that perform well in almost all the tests, but require, in particular the one based on sampling the discretized program, a substantial amount of memory.  

The paper is structured as follows: Section~\ref{sec:background} introduces the needed background knowledge.
Section~\ref{sec:hpasp} illustrates Hybrid Probabilistic Answer Set Programming and Section~\ref{sec:alg} presents the exact and approximate algorithms for performing inference, whose performance are tested in Section~\ref{sec:experiments}.
Section~\ref{sec:related} compares our proposal with related work and Section~\ref{sec:conclusions} concludes the paper.

\section{Background}
\label{sec:background}

\revisione{In this section we review the main concepts used in the paper.}

\subsection{\revisione{Answer Set Programming}}
\label{subsec:asp}
In this paper, we consider Answer Set Programming (ASP)~\citep{brewka2011asp}.
An answer set program contains \textit{disjunctive rules} of the form
$
h_1 ; \dots ; h_m \impl \ b_1, \dots, b_n.
$
where the disjunction of atoms $h_i$s is called \textit{head} and the conjunction of literals $b_i$s is called \textit{body}.
Rules with no atoms in the head are called \textit{constraints} while rules with no literals in the body and a single atom in the head are called \textit{facts}.
\textit{Choice rules} are a particular type of rules where a single head is enclosed in curly braces as in $\{a\} \impl b_1, \dots, b_n$, whose meaning is that $a$ may or may not hold if the body is true.
They are syntactic sugar for $a;not\_a \impl b_1,\dots,b_n$ where $not\_a$ is an atom for a new predicate not appearing elsewhere in the program.
We allow aggregates~\citep{alviano2018aggregates} in the body, one of the key features of ASP that allows representing expressive relations between the objects of the domain of interest.
An example of a rule $r0$ containing an aggregate is: $v(A)\impl \#count\{X:b(X)\}=A$.
Here, variable $A$ is unified with the integer resulting from the evaluation of the aggregate $\#count\{X:b(X)\}$, that counts the elements $X$ such that $b(X)$ is true.

\revisione{The semantics of ASP is based on the concept of \textit{stable model}~\citep{gelfond1988stable}.
Every answer set program has zero or more stable models, also called \textit{answer sets}.
An interpretation $I$ for a program $P$ is a subset of its Herbrand base.
The \textit{reduct} of a program $P$ with respect to an interpretation $I$ is the set of rules of the grounding of $P$ with the body true in $I$.
An interpretation $I$ is a stable model or, equivalently, an answer set, of $P$ if it is a minimal model (under set inclusion) of the reduct of $P$ w.r.t. $I$.
}
We denote with $AS(P)$ the set of answer sets of an answer set program $P$ and with $|AS(P)|$ its cardinality.
\revisione{If $AS(P) = \emptyset$ we say that $P$ is unsatisfiable.}
If we consider rule $r0$ and two additional rules, $\{b(0)\}$ and $b(1)$, the resulting program has 2 answer sets: $\{\{b(1), b(0), v(2)\},\{b(1), v(1)\}\}$.
Finally, we will also use the concept of \textit{projective solutions} of an answer set program $P$ onto a set of atoms $B$ (defined as in~\citep{gebser2009projective}): $AS_{B}(P) = \{A \cap B \mid A \in AS(P)\}$.

\subsection{\revisione{Probabilistic Answer Set Programming}}
\label{subsec:pasp}
Uncertainty in Logic Programming can be represented with discrete Boolean probabilistic facts of the form $\Pi::a$ where $\Pi \in [0,1]$ and $a$ is an atom that does not appear in the head of rules.
These are considered independent: this assumption may seem restrictive but, in practice, the same expressivity of Bayesian networks can be achieved by means of rules and extra atoms~\citep{Rig23-BKaddress}.
Every probabilistic fact corresponds to a Boolean random variable.
With probabilistic facts, a normal logic program becomes a \textit{probabilistic} logic program.
One of the most widely adopted semantics in this context is the Distribution Semantics (DS)~\citep{DBLP:conf/iclp/Sato95}.
Under the DS, each of the $n$ probabilistic facts can be included or not in a \textit{world} $w$, generating $2^n$ worlds.
\revisione{Every world is a normal logic program.}
The DS requires that every world has a unique model.
The probability of a world $w$ is defined as
$
P(w) = \prod_{f_i \in w} \Pi_i \cdot \prod_{f_i \not \in w} (1 - \Pi_i).
$

If we extend an answer set program with probabilistic facts, we obtain a \textit{probabilistic} answer set program that we interpret under the credal semantics (CS)~\citep{DBLP:conf/ilp/CozmanM16,cozman2020pasp}.
In the following, when we write ``probabilistic answer set program'', we assume that the program follows the credal semantics.
Similarly to the DS, the CS defines a probability distribution over worlds.
However, every world (which is an answer set program in this case) may have 0 or more stable models.
\revisione{A query $q$ is a conjunction of ground literals.}
The probability $P(q)$ of a query $q$ lies in the range $[\lowerprob(q), \upperprob(q)]$ where
\begin{equation}
\label{eq:lower_upper_prob}
\begin{split}
\lowerprob(q) &= \sum_{w_i \ \text{such that} \ \forall m \in AS(w_i), \ m \models q} P(w_i), \\
\upperprob(q) &= \sum_{w_i \ \text{such that} \ \exists m \in AS(w_i), \ m \models q} P(w_i).
\end{split}
\end{equation}
That is, the lower probability is given by the sum of the probabilities of the worlds where the query is true in every answer while the upper probability is given by the sum of the probabilities of the worlds where the query is true in at least one answer set. 
Here, as usual, we assume that all the probabilistic facts are independent and that they cannot appear as heads of rules.
\revisione{In other words, the lower (upper) probability is the sum of the probabilities of the worlds from which the query is a cautious (brave) consequence. }
Conditional inference \revisione{means computing upper and lower bounds for the probability of a query $q$ given evidence $e$, which is usually given as a conjunction of ground literls.}
The formulas are~\citep{cozman2017semantics}:
\begin{equation}
\label{eq:lower_upper_conditional}
\begin{split}
\lowerprob(q \mid e) &= \frac{\lowerprob(q,e)}{\lowerprob(q,e) + \upperprob(\lnot q,e)} \\
\upperprob(q \mid e) &= \frac{\upperprob(q,e)}{\upperprob(q,e) + \lowerprob(\lnot q,e)}    
\end{split}
\end{equation}
\revisione{Conditional probabilities are undefined if the denominator is 0.}


\begin{example}
\label{ex:pasp_example}
The following probabilistic answer set program has 2 probabilistic facts, $0.3::a$ and $0.4::b$:
\normalfont{
\begin{lstlisting}
0.3::a. 0.4::b.
q0 ; q1:- a. q0:- b.
\end{lstlisting}
}
There are $2^2 = 4$ worlds: $w_0 = \{not\ a, not\ b\}$, $w_1 = \{not\ a, b\}$, $w_2 = \{a, not\ b\}$, and $w_3 = \{a,b\}$ (with $not\ f$ we mean that the fact $f$ is absent, not selected), whose probability and answer sets are shown in Table~\ref{tab:world_as_pasp}.
For example, the answer sets for $w_2$ are computed from the following program: $a.\ q0 ; q1 \impl a.\ q0 \impl b.$
If we consider the query $q0$, it is true in all the answer sets of $w_1$ and $w_3$ and in one of those of $w_2$, thus we get $[P(w_1) + P(w_3), P(w_1) + P(w_2) + P(w_3)] = [0.4,0.58]$ as probability bounds. {$\square$}
\end{example}
DS and CS can be given an alternative but equivalent definition based on sampling in the following way.
We repeatedly sample worlds by sampling every probabilistic fact obtaining a normal logic program or an answer set program.
Then, we compute its model(s) and we verify whether the query is true in it (them).
The probability of the query under DS is the fraction of worlds in whose model the query is true as the number of sampled worlds tends to infinity.
For CS the lower probability is the fraction of worlds where the query is true in all models and the upper probability is the fraction of worlds where the query is true in at least one model.
We call this the \textit{sampling semantics} and it is equivalent to the DS and CS definitions by the law of large numbers.

\begin{table}[tb]
\centering
\caption{Worlds, probabilities, and answer sets for Example~\ref{ex:pasp_example}. }
\label{tab:world_as_pasp}
\begin{tabular}{|c | c | c | c | c|}
\ id \ & \ $a$ \ & \ $b$ \ & \ $P(w)$ \ & \ $AS(w)$ \ \\  
\hline
\hline
$w_0$ & 0 & 0 & $0.7\cdot 0.6 = 0.42$ & $\{\{\}\}$ \\
$w_1$ & 0 & 1 & $0.7 \cdot 0.4 = 0.28$ & $\{\{q0,b\}\}$ \\
$w_2$ & 1 & 0 & $0.3 \cdot 0.6 = 0.18$ & $\{\{a,q0\},\{a,q1\}\}$ \\
$w_3$ & 1 & 1 & $0.3 \cdot 0.4 = 0.12$ & $\{\{a,b,q0\}\}$ \\
\hline
\end{tabular} 
\end{table}

The complexity of the CS has been thoroughly studied by~\cite{cozman2020complexity}.
In particular, they focus on analyzing the cautious reasoning (CR) problem: given a PASP $P$, a query $q$, evidence $e$, and a value $\gamma \in \mathbb{R}$, the result is positive if $\lowerprob(q \mid e) > \gamma$, negative otherwise (or if $\lowerprob(e) = 0$).
A summary of their results is shown in Table~\ref{tab:complexity_pasp}.
For instance, in PASP with stratified negation and disjunctions in the head, the complexity of the CR problem is in the class $PP^{\Sigma_2^p}$, where $PP$ is the class of the problems that can be solved in polynomial time by a probabilistic Turing machine~\citep{doi:10.1137/0206049}.
The complexity is even higher if aggregates are allowed. 
\begin{table}[tb]
\centering
\caption{\revisione{Tight complexity bounds} of the CR problem in PASP from~\citep{cozman2020complexity}.
\textit{not stratified} denotes programs with stratified negations and $\vee$ disjunction in the head.
}
\label{tab:complexity_pasp}
\begin{tabular}{|c | c | c |} 
Language & Propositional & Bounded Arity \\  
\hline
\hline
\{\} & PP & $PP^{NP}$ \\
\{not stratified\} & PP & $PP^{NP}$ \\
\{not\} & $PP^{NP}$ & $PP^{\Sigma_2^p}$ \\
\{$\vee$\} & $PP^{NP}$ & $PP^{\Sigma_2^p}$ \\
\{not stratified, $\vee$\} & $PP^{\Sigma_2^p}$ & $PP^{\Sigma_3^p}$ \\
\{not, $\vee$\} & $PP^{\Sigma_2^p}$ & $PP^{\Sigma_3^p}$ \\
\hline
\end{tabular}
\end{table}

\subsection{\revisione{ProbLog and Hybrid ProbLog}}
\label{subsec:hybrid_problog}
\revisione{
A ProbLog~\citep{DBLP:conf/ijcai/RaedtKT07} program is composed by a set of Boolean probabilistic facts as described in Section~\ref{subsec:pasp} and a set of definite logical rules, and it is interpreted under the Distribution Semantics.
The probability of a query $q$ is computed as the sum of the probabilities of the worlds where the query is true: every world has a unique model so it is a sharp probability value.
}

\revisione{
\cite{DBLP:conf/ilp/GutmannJR10} proposed Hybrid ProbLog, an extension of ProbLog with continuous facts of the form
$$
(X,\phi)::b
$$
where $b$ is an atom, $X$ is a variable appearing in $b$, and $\phi$ is a special atom indicating the continuous distribution followed by $X$.
An example of continuous fact is $(X,gaussian(0,1))::a(X)$, stating that the variable $X$ in $a(X)$ follows a gaussian distribution with mean 0 and variance 1.  
A Hybrid ProbLog program $P$ is a pair $(R,T)$ where $R$ is a set of rules and $T = T^c \cup T^d$ is a finite set of continuous ($T^c$) and discrete ($T^d$) probabilistic facts.
The value of a continuous random variable $X$ can be compared only with constants through special predicates: $below(X,c_0)$ and $above(X,c_0)$, with $c_0 \in \mathbb{R}$, that succeed if the value of $X$ is respectively less than or greater than $c_0$, and $between(X,c_0,c_1)$, with $c_0, c_1 \in \mathbb{R}$, $c_0 < c_1$, that succeeds if $X \in [c_0,c_1]$.
}

\revisione{
The semantics of Hybrid ProbLog is given in~\citep{DBLP:conf/ilp/GutmannJR10} in a proof theoretic way.
A vector of samples, one for each continuous variable, defines a so-called continuous subprogram, so the joint distribution of the continuous random variables defines a joint distribution (with joint density $f(\mathbf{x})$) over continuous subprograms.
An interval $I \in \mathbb{R}^n$, where $n$ is the number of continuous facts, is defined as the cartesian product of an interval for each continuous random variable and the probability $P(\mathbf{X} \in I)$ can be computed by integrating $f(\mathbf{x})$ over $I$, i.e.,
$$
P(\mathbf{X} \in I) = \int_{I} f(\mathbf{x}) \ d\mathbf{x}.
$$
Given a query $q$, an interval $I$ is called \textit{admissible} if, for any $\mathbf{x}$ and $\mathbf{y}$ in $I$ and for any truth value of the probabilistic facts, the truth value of the $q$ evaluated in the program obtained by assuming $\mathbf{X}$ takes value $\mathbf{x}$ and $\mathbf{y}$ is the same.
In other words, inside an admissible interval, the truth value of a query is not influenced by the values of the continuous random variables i.e., it is always true or false once the value of the discrete facts is fixed.
For instance, if we have the continuous fact
$(X,gaussian(0,1))::a(X)$ 
and a rule 
$q \impl a(X), between(X,0,2)$,
the interval $[0,3]$ is not admissible for the query $q$ while $[0.5,1.5]$ is (since for any value of $X \in [0.5,1.5]$, $between(X,0,2)$ is always true).
A partition is a set of intervals and it is called admissible if every interval is admissible.
The probability of a query is defined by extracting its proofs together with the probabilistic facts, continuous facts, and comparison predicates used in the proofs.
These proofs are finitely many, since the programs do not include function symbols, and generalize the infinitely many proofs of a hybrid program (since there can be infinitely many values for continuous variables).
\cite{DBLP:conf/ilp/GutmannJR10} proved that an admissible partition exists for each query having a finite number of proofs and the probability of a query does not depend on the admissible partition chosen.
They also propose an inference algorithm that first computes all the proofs for a query.
Then, according to the continuous facts and comparison predicates involved, it identifies the needed partitions.
To avoid counting multiple times the same contribution, the proofs are made disjoint.
Lastly, the disjoint proofs are converted into a binary decision diagram (BDD) from which the probability can be computed by traversing it bottom up (i.e., with the algorithm adopted in ProbLog~\citep{DBLP:conf/ijcai/RaedtKT07}). 
}

\revisione{
Let us clarify it with a simple example.
}
\begin{example}
\revisione{
The following Hybrid ProbLog program has one discrete fact and one continuous fact.
}
\begin{lstlisting}
0.4::a.
(X,gaussian(0,1))::b(X).
q:- a.
q:- b(X), above(X,0.3).
\end{lstlisting}
\revisione{
Consider the query $q$.
It has two proofs: $a$ and $b(X), above(X,0.3)$.
}
\revisione{
The probability of $q$ is computed as $P(a) \cdot P(X > 0) + P(a) \cdot (1 - P(X > 0)) + (1 - P(a)) \cdot P(X > 0) = 0.4\cdot 0.5 + 0.4 \cdot 0.5 + 0.6 \cdot 0.5 = 0.7$. $\square$
}
\end{example}

\revisione{
The Hybrid ProbLog framework imposes some restrictions on the continuous random variables and how they can be used in programs, namely:
i) comparison predicates must compare variables with numeric constants,
ii) arithmetic expressions involving continuous random variables are not allowed, and 
iii) terms iside random variables definitions can be used only inside comparison predicates.
These restrictions allow the algorithm for the computation of the probability to be relatively simple.
}



\section{Hybrid Probabilistic Answer Set Programming}
\label{sec:hpasp}
Probabilistic logic programs that combine both discrete and continuous random variables are usually named \textit{hybrid}\footnote{In ASP terminology, the word ``hybrid'' is usually adopted to describe an extension of ASP, such as the one by~\citep{janhunen2017clingolinear}. Here we use the word hybrid exclusively to denote the presence of discrete and continuous random variables.}.
In this paper, we adopt the same adjective to describe probabilistic answer set programs with both discrete and continuous random variables, thus introducing \textit{hybrid} probabilistic answer set programs.
We use the acronym (H)PASP to indicate both (Hybrid) Probabilistic Answer Set Programming and (hybrid) probabilistic answer set program; the meaning will be clear from the context.

Without loss of generality, we consider only ground discrete and continuous probabilistic facts.
With the syntax
\begin{lstlisting}
f : distribution.
\end{lstlisting}
where $f$ is ground, we indicate a continuous random variable $f$ that follows the distribution specified by $distribution$.
For example, to state that $a$ is a continuous random variable that follows a Gaussian distribution with mean 2 and variance 1 we write
\begin{lstlisting}
a : gaussian(2,1).
\end{lstlisting}

\begin{definition}
A \textit{hybrid} probabilistic answer set program is a triple $(D,C,R)$ where $D$ is a finite set of ground independent probabilistic facts, $C$ is a finite set of continuous random variables definitions, and $R$ is a set of rules with none of the atoms in $D$ or $C$ in the head. $\square$
\end{definition}

As in Hybrid ProbLog, we reserve some special predicates (that we call \textit{comparison predicates}) to compare the value of the random variables with numeric ($\in \mathbb{R}$) constants: 
$\mathit{above}(\mathit{Var},\mathit{value})$, which is true if the value for the variable $\mathit{Var}$ is greater than the numeric value $\mathit{value}$ ($\mathit{Var} > \mathit{value}$);
$\mathit{below}(\mathit{Var},\mathit{value})$ which is true if the value for the variable $\mathit{Var}$ is less than the numeric value $\mathit{value}$ ($\mathit{Var} < \mathit{value}$);
$\mathit{between}(\mathit{Var},l,u)$ which is true if the value for the variable $\mathit{Var}$ is between the range defined by the numeric values $l$ and $u$ (i.e., $\mathit{Var} > l$ and $\mathit{Var} < u$);
$\mathit{outside}(\mathit{Var},l,u)$ which is true if the value for the variable $\mathit{Var}$ is outside the range defined by the numeric values $l$ and $u$ (i.e., $\mathit{Var} < l$ or $\mathit{Var} > u$).
Given a value for a random variable, we can evaluate the atoms for the comparison predicates.
For example, if variable $a$ has value 1.2, $\mathit{below}(a,1.8)$ is true, $\mathit{above}(a,1.8)$ is false, $\mathit{between}(a,1.1,1.8)$ is true, and $\mathit{outside}(a,1.1,1.8)$ is false.
In our framework, we consider the same restrictions of~\citep{DBLP:conf/ilp/GutmannJR10}.
\revisione{Note that $outside/3$ literals are not allowed by Hybrid ProbLog.}

We can extend the sampling semantics for the CS to hybrid probabilistic answer set programs: we sample every probabilistic fact and a value for every continuous variable, \revisione{obtaining sample $s=(d,c)$, where $d$ is the sample for discrete facts and $c$ is the sample for contiunous random variables.
From $s$, we build a \textit{hybrid world} $w=(P,c)$ where $P$ is obtained by adding to $R$ each fact $f_i$ for which $p_i::f_i \in D$ and $d_i=1$ in vector $d$.} 
We then ground the rules and check the sampled values for the continuous random variables against every comparison predicate that contains the variable: 
i) if the comparison is false, we remove the grounding of the rule or 
ii) if the comparison is true, we remove the comparison predicate from the body of the ground rule.
We are left with an answer set program $w^{'}$, and we can compute its stable models.
The lower and upper probability of a query $q$ are given, as in the sampling semantics for CS, as the fraction of worlds where the query is true in every answer set and in at least answer set, respectively, as the number of samples goes to infinity.
That is,
\begin{equation}
\label{eq:lower_upper_prob_sampling}
\begin{split}
  \lowerprob(q) &= \lim_{n \to \infty} \frac{\sum_{i = 1}^n \mathbbm{1}(\forall m \in AS(w_{i}^{'}), m \models q)}{n} \\
  \upperprob(q) &= \lim_{n \to \infty} \frac{\sum_{i = 1}^n \mathbbm{1}(\exists m \in AS(w_{i}^{'}), m \models q)}{n} \\
\end{split}
\end{equation}
where $\mathbbm{1}$ is the indicator function, returning 1 if its argument is true, 0 otherwise, and $AS(w_{i}^{'})$ is the set of answer sets of the program $w_{i}^{'}$ obtained from the $i$-th sample.
We call this the \textit{hybrid sampling semantics}.

To better specify the syntax, let us introduce an example.
\begin{example}
\label{ex:blood_pressure}
Consider a medical domain where we are interested in computing the probability of stroke given the values of the blood pressure of some individuals.
{
\normalfont
\begin{lstlisting}
0.4::pred_d(1..4).
0.6::pred_s(1..4).
d(1..4):gamma(70,1).
s(1..4):gamma(120,1).

prob_d(P):- outside(d(P), 60, 80).
prob_s(P):- outside(s(P), 110, 130).

prob(P):- prob_d(P), pred_d(P).
prob(P):- prob_s(P), pred_s(P).

stroke(P);not_stroke(P):- prob(P).

:- #count{X:prob(X)}=P, 
   #count{X:stroke(X),prob(X)}=S,
   10*S < 4*P.

high_number_strokes:- 
   #count{X : stroke(X)}=CS, CS > 1.
\end{lstlisting}
}
For ease of explanation and conciseness, we use the syntax $a(l..u)$ with $l, u \in \mathbb{N}$, $l < u$, to denote a set of atoms $a(l), a(l+1), \dots, a(u)$.
The program considers 4 people indexed with 1, 2, 3, and 4.
The first two lines introduce 8 discrete probabilistic facts, $pred\_d(i)$, $i \in \{1,\dots,4\}$ (pred is short for predisposition) which are true with probability 0.4 and $pred\_s(i)$, $i \in \{1,\dots,4\}$ which are true with probability 0.6, and 8 continuous random variables ($d(1), \dots, d(4), s(1), \dots, s(4)$ where $d$ stands for diastolic and $s$ for systolic) where the first four follow a gamma distribution with shape 70 and rate 1 and while the remaining  follow a gamma distribution with shape 120 and rate 1.
$prob\_d/1$ indicates a problem from the diastolic pressure and it is true if there is a problem coming from the diastolic pressure if its value is below 60 or above 80.
Similarly for $prob\_s/1$.
In general, a person has a blood pressure problem ($prob/1$) she has a diastolic or systolic pressure problem and there is a predisposition for it (either $pred\_d$ or $pred\_s$).
A person having a blood pressure problem can have a stroke or not.
The constraint states that at least 40\% of people that have a pressure problem also have a stroke.
Finally, we may be interested in computing the probability that more than one person has a stroke ($high\_number\_strokes/0$). {$\square$}
\end{example}

The comparison predicates subdivide the domain of a random variable into disjoint and exhaustive intervals $I_1 \cup I_2 \cup \dots \cup I_m$.
The extremes of the disjoint intervals are obtained by selecting all constants appearing in comparison predicates for continuous random variables, removing the duplicates, and ordering them in increasing order.
In this way we obtain a list of values $[b_1\dots,b_{m+1}]$ where $b_1=-\infty$ and $b_{m+1}=+\infty$ such that $I_k=[b_k,b_{k+1}]$ for $k=1,\dots,m$.
This process is described in Algorithm 2 of~\citep{DBLP:conf/ilp/GutmannJR10}.
\begin{example}
\label{ex:example_interval}
Consider the following simple program: 
\normalfont{
\begin{lstlisting}
0.4::b. a:gaussian(0,1).
q0 ; q1:- below(a,0.5).
q0:- below(a,0.7), b.
\end{lstlisting}
}
\noindent
There are 3 intervals to consider: $I^a_1 = ]-\infty,0.5]$, $I^a_2 = [0.5,0.7]$, and $I^a_3 = [0.7,+\infty[$. {$\square$}
\end{example}

\begin{table}[tb]
\caption{Worlds and probabilities for Example~\ref{ex:example_interval}. $f_1$ and $f_2$ are the two probabilistic facts obtained by the discretization process of the continuous probabilistic fact $a$ into three intervals. 
The column LP/UP indicates whether the world contributes only to the upper probability (UP) or also to the lower probability (LP + UP).
}
\label{tab:tab_ex_interval}
\centering
\begin{tabular}{|c | c | c | c | c | c|} 
\ id \ & \ $b$ \ & \ $f_1$ \ & \ $f_2$ \ & \ LP/UP \ & \ $P(w)$ \ \\  
\hline
\hline
0 & 0 & 0 & 0 & -       & 0.145 \\
1 & 0 & 0 & 1 & -       & 0.040 \\
2 & 0 & 1 & 0 & UP      & 0.325 \\
3 & 0 & 1 & 1 & UP      & 0.089 \\
  4 & 1 & 0 & 0 & -       & 0.097 \\
  5 & 1 & 0 & 1 & LP + UP & 0.027 \\
  6 & 1 & 1 & 0 & LP + UP & 0.217 \\
  7 & 1 & 1 & 1 & LP + UP & 0.060 \\
\end{tabular}
\end{table}

We transform a HPASP $P_c$ into a PASP $P_c^d$ via a process that we call \textit{discretization}.
The probability of a query $q$ in $P_c$ is the same as the probability asked in $P_c^d$. 
Thus, using discretization, inference in HPASP is performed using inference in PASP.
We now show how to generate the discretized program and then we prove the correctness of the transformation.



We need to make the rules involving comparison predicates considering more than one interval disjoint, to avoid counting multiple times the same contribution.
To do so, we can first convert all the comparison predicates $\mathit{between}/3$ and $\mathit{outside}/3$ into a combination of $\mathit{below}/2$ and $\mathit{above}/2$.
That is, every rule, 
$$
h\impl \ l_1,\dots,\mathit{between}(a,lb,ub),\dots,l_n.
$$
is converted into 
$$
h\impl \ l_1,\dots,\mathit{above}(a,lb), \mathit{below}(a,ub),\dots,l_n.
$$
The conversion of $\mathit{outside}(a,lb,ub)$ requires introducing two rules.
That is, 
$$
h\impl \ l_1, \dots, \mathit{outside}(a,lb,ub), \dots,l_n.
$$
requires generating two rules
\begin{equation*}
\begin{split}
h_a&\impl \ l_1, \dots, \mathit{above}(a,ub), \dots, l_n. \\ 
h_b&\impl \ l_1, \dots, \mathit{below}(a,lb), \dots, l_n.
\end{split}
\end{equation*}
If there are multiple comparison predicates in the body of a rule, the conversion is applied multiple times, until we get rules with only $\mathit{above}/2$ and $\mathit{below}/2$ predicates.
After this, for every continuous variable $f_j$ with $m$ associated intervals we introduce a set of clauses~\citep{DeR-NIPS08,DBLP:conf/ilp/GutmannJR10}
\begin{align}
\label{eq:aux_clauses}
\begin{split}
&h_1^j \impl \ f_{j1}. \\
&h_2^j \impl \ not\ f_{j1}, f_{j2}.\\
&\dots \\
&h_{m-1}^j \impl \ not\ f_{j1}, not \ f_{j2}, \dots, f_{jm-1}. \\
&h_m^j \impl \ not\ f_{j1}, not \ f_{j2}, \dots, not \ f_{jm-1}.  
\end{split}
\end{align}
where each $h_i^j$ is a propositional atom that is true if the continuous variable $f_j$ takes value in the interval $I_i^{f_j}$ and each $f_{jk}$ is a fresh probabilistic fact with probability $\pi_{jk}$ for $k=1,\dots,m-1$ computed as
\begin{equation}
\label{eq:probability_computation_facts}
\pi_{jk} = \revisione{\frac{P(b_k \leq f_j \leq b_{k+1})}{1 - P(f_j \leq b_k)}} = \frac{\int_{b_k}^{b_{k+1}} p_j(x_j) \ dx_j}{1 - \int_{-\infty}^{b_k} p_j(x_j) \ dx_j} 
\end{equation}
where $p_j(x_j)$ is the probability density function of the continuous random variable $f_j$ and $b_k$ for $k=1,\dots,m$ are the bounds of the intervals.
After this step, we identify the comparison atoms that are true in more than one interval.
A clause containing a comparison atom $below(f_j,b_{k+1})$ is replicated $k$ times, once for each of the intervals $I_l$ with $l=1,\dots,k$.
The comparison atom of the $k$-th replicated clause is replaced by $h_k^j$.
Similarly for $above/2$.
If a clause contains comparison atoms on multiple variables, this process is repeated multiple times.
That is, if a clause contains $n_c$ comparison atoms on the variables $v_1,\dots,v_n$ that are true in $k_1,\dots,k_n$ intervals, we get $k_1 \times \dots \times k_n$ clauses.
Note that the complexity of inference is dominated by the number of probabilistic facts, rather than the number of clauses.

Algorithm~\ref{alg:discretize} depicts the pseudocode for the discretization process applied to a HPASP with ground rules $R$, continuous random variable definitions $C$, and discrete probabilistic facts $D$.
\revisione{It is composed of three main functions, \textsc{ConvertBetweenOutside}, \textsc{HandleIntervals}, and \textsc{HandleComparisonAtoms} that convert the between and outside comparison predicates, compute the number of intervals and the new discrete probabilistic facts, and manage the comparison atoms, respectively.}
Functions \textsc{ConvertBetween} and \textsc{ConvertOutside} convert respectively the comparison predicates $between/3$ and $outside/3$ into combinations of $above/2$ and $below/2$.
Function \textsc{ComputeIntervals}($c$) computes the intervals for the continuous random variable $c$.
Function \textsc{ComputeProbability}($c,i$) computes the probability for the $i$-th probabilistic fact for the continuous random variable $c$.
Function \textsc{GetComparisonAtoms}($R$) returns all the comparison atoms in the rules $R$ of the program.
Function \textsc{ComputeIntervalsComparison}($c$) computes the number of intervals involved in the comparison atom $c$.
Function \textsc{Replace}($a,h_i$) replaces the comparison atom $a$ with the corresponding discrete fact $h_i$ representing the $i$-th interval.

\begin{theorem}
\revisione{
Consider an hybird probabilistic answer set program $P$.
Let
$n_r$ be the number of rules,
$n_c$ be the number of continuous facts,
$n_k$ be the number of comparison atoms in the program,
$n_i$ be the maximum number of intervals for a continuous fact,
$o$ be the maximum number of $outside/3$ atoms in the body of a rule, and
$i_c$ be the maximum number of intervals where a comparison atom is true.
Then, Algorithm~\ref{alg:discretize} requires $O(n_r^o + n_c \cdot n_i + n_k \cdot n_r \cdot i_c)$ time.}
\end{theorem}
\begin{proof}
\revisione{
Let us analyze the three functions used in \textsc{Discretize}.
For function \textsc{ConvertBetweenOutside}, the body of the loop at lines~\ref{alg_line:start_convert_between}--\ref{alg_line:end_convert_between} is executed $n_r$ times.
The while loop at lines~\ref{alg_line:start_while_outside}--\ref{alg_line:end_while_outside} is executed $n_r^o$ times.
Thus, function \textsc{ConvertBetweenOutside} has complexity $O(n_r^o)$.
Function \textsc{HandleIntervals} has complexity $O(n_c \cdot n_i)$.
In function \textsc{HandleComparisonAtoms}, the loop at lines~\ref{alg_line:loop_cp}--\ref{alg_line:loop_cp_end} is executed $n_k$ times.
For each of them, the loop at lines~\ref{alg_line:loop_pf}--\ref{alg_line:loop_pf_end} is executed $n_r$ times.
Lastly, the innermost loop at lines~\ref{alg_line:loop_intervals}--\ref{alg_line:loop_intervals_end} is executed at most $i_c$ times.
So, function \textsc{HandleComparisonAtoms} has complexity $O(n_k \cdot n_r \cdot i_c)$.
Overall, function \textsc{Discretize} has complexity $O(n_r^o + n_c \cdot n_i + n_k \cdot n_r \cdot i_c)$.}
\end{proof}
Note that the complexity of inference in probabilistic answer set programs is very high (see Table~\ref{tab:complexity_pasp}), so the discretization process has a small impact on the overall process.

\begin{algorithm}[t]
\begin{scriptsize}
\caption{Function \textsc{Discretize}: discretization of a hybrid probabilistic answer set program with rules $R$, continuous probabilistic facts $C$, and discrete probabilistic facts $D$.}
\label{alg:discretize}
\begin{algorithmic}[1]
\Function{Discretize}{$R,C,D$}
\State $R^d \gets$ \Call{ConvertBetweenOutside}{$R$}
\State $R^d, D \gets$ \Call{HandleIntervals}{$R^d,C,D$}
\State $R^d \gets$ \Call{HandleComparisonAtoms}{$R^d$}
\State \Return $R^d, D$
\EndFunction
\\ \hrulefill
\Function{ConvertBetweenOutside}{$R$}
  \State $R^d \gets \{\}$
  \ForAll{$r \in R$} \label{alg_line:start_convert_between}
    \State $r_c \gets$ \Call{ConvertBetween}{$r$}
    \State $R^d \gets R^d \cup r_c$
  \EndFor \label{alg_line:end_convert_between}
  \While{$\exists r \in R^d : outside/3 \in r$} \label{alg_line:start_while_outside}
    \State $r_a, r_b \gets$ \Call{ConvertOutside}{$r$}
    \State $R^d \gets R^d \setminus \{r\} \cup \{r_a, r_b\}$
  \EndWhile \label{alg_line:end_while_outside}
  \State \Return $R^d$
\EndFunction
\\ \hrulefill
\Function{HandleIntervals}{$R^d,C,D$}
  \ForAll{$c \in C$}
    \State $n_i \gets$ \Call{ComputeIntervals}{$c$} \Comment{Computation of the intervals for the continuous probabilistic fact $c$.}
    \For{$i \in \{1,\dots,n_i\}$}
        \State $f_i \gets$ \Call{ComputeProbability}{$c,i$} \Comment{$f_i$ is a fresh probabilistic fact for the current interval $i$.}
        \State $D \gets D \cup f_i$
        \State $R^d \gets R^d \cup \{h_i \leftarrow \bigwedge_{j = 1}^{i-1} not\ f_j\} \land f_i$
    \EndFor
  \EndFor
  \State \Return $R^d, D$
\EndFunction
\\ \hrulefill
\Function{HandleComparisonAtoms}{$R$}
\State $K \gets$\Call{GetComparisonAtoms}{$R$}
\ForAll{$a \in K$ } \Comment{Loop over all the comparison atoms.} \label{alg_line:loop_cp}
\State $i_c \gets$ \Call{ComputeIntervalsComparison}{$a$}
  \State $R^t \gets \{\}$
  \ForAll{$r \in R$} \label{alg_line:loop_pf}
    \If{$a \in R$}
      \For{$i \in \{1,\dots,i_a\}$} \Comment{Loop over the intervals that make the comparison atom $a$ true.} \label{alg_line:loop_intervals}
        \State $r_a \gets$\Call{Replace}{$r, a,h_i$}
        \State $R^t \gets R^t \cup \{r_a\}$
      \EndFor \label{alg_line:loop_intervals_end}
    \Else
      \State $R^t \gets R^t \cup \{r\}$
    \EndIf
  \EndFor \label{alg_line:loop_pf_end}
  \State $R \gets R^t$
\EndFor \label{alg_line:loop_cp_end}
\State \Return $R$
\EndFunction


\end{algorithmic}
\end{scriptsize}
\end{algorithm}


We can perform conditional inference using the formulas for discrete programs (Equation~\ref{eq:lower_upper_conditional}): if the evidence is on a discrete variable, we can directly apply that formula to the discretized program.
If the evidence is on a continuous variable, say $e = (X > k)$ with $k \in \mathbb{R}$, we first need to create a discretized version of the program by also taking into account this constraint.

To clarify, in Example~\ref{ex:example_interval} we have a variable $a$ with two numerical constraints on it: $\mathit{below}(a,0.5)$ and $\mathit{below}(a,0.7)$.
The first interval, $I_1^a$, is $]-\infty,0.5]$, the second, $I_2^a$, is $[0.5,0.7]$, and the third $[0.7,\infty[$.
After inserting the new probabilistic facts, $f_{a1}$ for $I_1^a$ and $f_{a2}$ for $I_2^a$, we add two more rules $h_1^a \impl f_{a1}.$ and $h_2^a \impl not \ f_{a1}, f_{a2}.$, we replicate the clauses with comparison predicates spanning more than one interval and replace the comparison predicates with the appropriate $h_i^a$ atom, to make them disjoint, as previously above.
For example, the comparison atom $\mathit{below}(a,0.7)$ of Example~\ref{ex:example_interval} is true in the intervals $I_1^a$ and $I_2^a$.
The clause containing it is duplicated and, in the first copy, we replace $\mathit{below}(a,0.7)$ with $h_1^a$, while in the second with $h_2^a$.
The other comparison atom, $\mathit{below}(a,0.5)$, is true in only one interval, so it is sufficient to replace it with $h_1^a$. 
This process is repeated for every continuous random variable.
Overall, we obtain the program shown in Example~\ref{ex:running_discretized}.
\begin{example}
\label{ex:running_discretized}
By applying the discretization process to Example~\ref{ex:example_interval}, we get that $\pi_{a1} = 0.6915$ and $\pi_{a2} = \frac{0.066}{1 - 0.6915} = 0.2139$.
The program becomes:
\normalfont
{
\begin{lstlisting}
0.4::b.
0.6915::a1.
0.2139::a2.
int1 :- a2.
int2 :- not a1, a2.
q0 ; q1:- int1.
q0:- int1, b.
q0:- int2, b.
\end{lstlisting}
}
We can now compute the probability of the query, for example, $q0$: $[\lowerprob(q0),\upperprob(q0)] = [0.303,0.718]$.
The worlds are shown in Table~\ref{tab:tab_ex_interval}.
Similarly, for the program of Example~\ref{ex:blood_pressure}, we get $P(high\_number\_strokes) = [0.256, 0.331]$. {$\square$}
\end{example}

We now show that the lower and upper probability of a query from the discretized program is the same as that from the hybrid sampling semantics.
\begin{theorem}
Given a query $q$, a hybrid program $P$, and its discretized version $P^d$, \revisione{the lower and upper probability of $q$ computed in the hybrid program ($\lowerprob(q)$ and $\upperprob(q)$) coincide with the lower and upper probability computed in its discretized version ($\lowerprob^{P^d}(q)$ and $\upperprob^{P^d}(q)$), i.e.,}
\begin{equation}
\begin{split}
\lowerprob(q) &= \lowerprob^{P^d}(q) \\
\upperprob(q) &= \upperprob^{P^d}(q)  
\end{split}
\end{equation}
\end{theorem}

\begin{proof}
Given a hybrid world $w$ and a clause $c$ that contains at least one comparison atom in the grounding of $P$, call $g_1,\dots,g_n$ the set of clauses in the grounding of $P^d$ generated from $c$.
There are two cases.
The first is that the samples for the continuous random variables do not satisfy the constraints in the body of $c$.
In this case, all the $g_i$ clauses will have a false body.
In the second case, there will be a single clause $g_i$ where all the $h^j_i$ atoms in the body are true.
These atoms can be removed, obtaining an answer set program $w^{'}$ that is the same as the one that would be obtained in the hybrid sampling semantics.
It remains to prove that the resulting probability distributions over the discretized worlds are the same.
We show that the probability of obtaining a program $w^{'}$ in the hybrid sampling semantics and in the sampling semantics from the discretized program is the same.
We need to prove that the probability of a continuous random variable $f_j$ of falling into the interval $I_k^{f_j}$ is the probability that in a world sampled from the discretized program atom $h_k^j$ is true.
The latter probability depends only on the probabilities of the $f_{jl}$ facts.
Thus it can be computed by observing that the $f_{jl}$ facts are mutually independent: since the part of the program defining $h_k^j$ is a stratified normal program, the lower and upper probabilities of $h_k^j$ coincide and, if $-\infty,b_2,\dots,b_{m-1},+\infty$ are the different bounds for the intervals they can be computed as:
\begin{equation}
\begin{split}
P(h_k^j) &= 
        P(not\ f_{j1}) \cdot 
        P(not\ f_{j2}) \cdot \dots \cdot 
        P(not f_{jk-1}) \cdot P(f_{jk}) \\
    & = \left(1 - \int_{-\infty}^{b_2} p_j(f_j)\ df_j\right) \cdot 
        \frac{1 - \int_{b_2}^{b_3} p_j(f_j) \ df_j}{1 - \int_{-\infty}^{b_2} p_j(f_j) \ df_j} \cdot \dots \cdot 
        \frac{\int_{b_k}^{b_{k+1}}  p_j(f_j) \ df_j}{1 - \int_{-\infty}^{b_k}  p_j(f_j) \ df_j}\\ 
    & = \left(1 - \int_{-\infty}^{b_k}  p_j(f_j) \ df_j\right) \cdot 
        \frac{\int_{b_k}^{b_{k+1}}  p_j(f_j) \ df_j}{1 - \int_{-\infty}^{b_k} p_j(f_j) \ df_j } \\
    & = \int_{b_k}^{b_{k+1}}  p_j(f_j) \ df_j
\end{split}
\end{equation}
\end{proof}

With our framework it is possible to answer queries in programs where variables follow a mixture of distributions.
For instance, in the following program
\begin{lstlisting}
0.4::c.
a:gaussian(10,3).
b:gaussian(9,2).
q0:- c, above(a,6.0). 
q0:- not c, above(b,6.0).  
\end{lstlisting}
we can compute the probability of the query $q0$ ($[\lowerprob(q0),\upperprob(q0)] = [0.923,0.923]$).
Here, if $c$ is true, we consider a gaussian distribution with mean 10 and variance 3, if $c$ is false, a gaussian distribution with mean 9 and variance 2.
In other words, we consider a variable that follows the first gaussian distribution if $c$ is true and the second gaussian distribution if $c$ is false.


One of the key features of ASP is the possibility of adding logical constraints that cut some of the possible answer sets.
However, when we consider a PASP (also obtained via discretization), constraints may cause loss of probability mass due to unsatisfiable worlds, since these do not contribute neither to the lower nor to the upper bound (Equation~\ref{eq:lower_upper_prob}). 
There can be constraints involving only discrete probabilistic facts, constraints involving only comparison atoms (thus, continuous random variables), and constraints involving both.
Let us discuss this last and more general case with an example.
\begin{example}
\label{ex:normalization_inconsistent_cont_discr}
Consider the program of Example~\ref{ex:example_interval} with the additional rule (constraint) $\impl b, below(a, 0.2)$.
When $b$ is true, the value of $a$ cannot be less than 0.2.
This results in a loss of probability mass in the discretized program.
The introduction of the constraint requires a more granular partition and an additional interval: $I^a_1 = ]-\infty,0.2]$, $I^a_2 = [0.2,0.5]$, $I^a_3 = [0.5,0.7]$, and $I^a_4 = [0.7,+\infty[$.
Note that the probabilities of the discretized facts will be different from the ones of Example~\ref{ex:example_interval}. {$\square$}
\end{example}
When every world is satisfiable, we have that~\citep{cozman2020pasp}:
\begin{equation}
\label{eq:sum_lower_upper}
\lowerprob(q)+\upperprob(not\ q) = 1  
\end{equation}
When at least one world is unsatisfiable, Equation~\ref{eq:sum_lower_upper} does not hold anymore, but we have $\lowerprob(q)+\upperprob(not\ q) + P(inc)=1$, where $P(inc)$ is the probability of the unsatisfiable worlds.
So we can still use the semantics but we need to provide the user also with $P(inc)$ beside $\lowerprob(q)$ and $\upperprob(q)$.
If we want to enforce Equation~\ref{eq:sum_lower_upper}, we can resort to normalization: we divide both the lower and the upper probability bounds of a query by the probability of the satisfiable worlds.
That is, call
\begin{equation}
\label{eq:normalization}
Z = \sum_{w_i \mid |AS(w_i)| > 0} P(w_i)
\end{equation}
then $P^n(q) = [\lowerprob^n(q),\upperprob^n(q)]$ with $\lowerprob^n(q) = \lowerprob(q) / Z, \ \upperprob^n(q) = \upperprob(q) / Z$.
This approach has the disadvantage that if $Z$ is 0 the semantics is not defined.
In Example~\ref{ex:normalization_inconsistent_cont_discr}, the normalizing factor $Z$ is 0.7683, thus the probability of $q0$ now is: $P(q0) = [0.093,0.633]$.
Note that the new bounds are not simply the bounds of Example~\ref{ex:running_discretized} divided by $Z$, since the constraint splits the domain even further, and the probabilities associated with the probabilistic facts obtained via discretization change.

\section{Algorithms}
\label{sec:alg}
In this section, we describe two exact and two approximate algorithms for performing inference in HPASP.

\subsection{Exact Inference}
\label{subsec:exact_inference}
\paragraph{\revisione{\textbf{Modifying the PASTA solver.}}}
We first modified the PASTA solver~\citep{AzzBellRig2022PASTA}\footnote{Code and datasets available on GitHub at \url{https://github.com/damianoazzolini/pasta} and on zenodo at \url{https://doi.org/10.5281/zenodo.11653976}.}, by implementing the conversion of the hybrid program into a regular probabilistic answer set program (Algorithm~\ref{alg:discretize}).
PASTA performs inference on PASP via projected answer set enumeration.
In a nutshell, to compute the probability of a query (without loss of generality assuming here it is an atom), the algorithm converts each probabilistic fact into a choice rule.
Then, it computes the answer sets projected on the atoms for the probabilistic facts and for the query.
In this way, PASTA is able to identify the answer set pertaining to each world.
For each world, there can be 3 possible cases:
i) a single projected answer set with the query in it, denoting that the query is true in every answer set, so this world contributes to both the lower and upper probability;
ii) a single answer set without the query in it: this world does not contribute to any probability; 
iii) two answer sets, one with the query and one without: the world contributes only to the upper probability.
There is a fourth implicit and possible case: if a world is not present, this means that the ASP obtained from the PASP by fixing the selected probabilistic facts is unsatisfiable.
Thus, in this case we need to consider the normalization factor of Equation~\ref{eq:normalization}.
PASTA already handles this by keeping track of the sum of the probabilities of the computed worlds.
The number of generated answer sets depends on the number of Boolean probabilistic facts and on the number of intervals for the continuous random variables, since every interval requires the introduction of a new Boolean probabilistic fact.
Overall, if there are $d$ discrete probabilistic facts and $c$ continuous random variables, the total number of probabilistic facts (after the conversion of the intervals) becomes $T = d + \sum_{i = 1}^{c}(k_i - 1)$, where $k_i$ is the number of intervals for the $i$-th continuous fact.
Finally, the total number of generated answer sets is bounded above by $2^{T+1}$, due to the projection on the probabilistic facts.
Clearly, generating an exponential number of answer sets is intractable, except for trivial domains.

\begin{example}[Inference with PASTA.]
\label{ex:inference_pasta}
Consider the discretized program described in Example~\ref{ex:running_discretized}.
It is converted into
\begin{lstlisting}
{a1}.
{a2}.
{b}.
int1 :- a2.
int2 :- not a1, a2.
q0 ; q1:- int1.
q0:- int1, b.
q0:- int2, b.
\end{lstlisting}
It has 10 answer sets projected on the atoms $q0/0$ and $a0/0$, $a1/0$, and $b/0$:
$AS_{1} = \{\}$,
$AS_{2} = \{a1\}$,
$AS_{3} = \{a2\}$,
$AS_{4} = \{a2,q0\}$,
$AS_{5} = \{b\}$,
$AS_{6} = \{a1,a2\}$,
$AS_{7} = \{a1,b\}$,
$AS_{8} = \{a2,b,q0\}$,
$AS_{9} = \{a1,a2,q0\}$, and 
$AS_{10} = \{a1,a2,b,q0\}$.
For example, the world where only $a2$ is true is represented by the answer sets $AS_3$ and $AS_4$.
The query $q0$ is present only in one of the two, so this world contributes only to the upper probability.
The world where $a1$ and $a2$ are true and $b$ is false is represented only by the answer set $AS_9$: the query is true in it so this world contributes to both the lower and upper probability.
The world where only $a1$ is true is represented by $AS_2$, $q0$ is not present in it so it does not contribute to the probability.
By applying similar considerations for all the worlds, it is possible to compute the probability of the query $q0$. {$\square$}
\end{example}

\paragraph{\revisione{\textbf{Modifying the aspcs solver.}}}
We also added the conversion from HPASP to PASP to the aspcs solver~\citep{AzzRig2023-AIXIA-IC}, built on top of the aspmc solver~\citep{DBLP:conf/kr/EiterHK21,DBLP:journals/tplp/KieselTK22}, that performs inference on Second Level Algebraic Model Counting (2AMC) problems, an extension of AMC~\citep{10.1016/j.jal.2016.11.031}.
Some of them are: MAP inference~\citep{DBLP:conf/ilp/ShterionovRVKMJ14}, decision theory inference in probabilistic logic programming~\citep{DBLP:conf/aaai/BroeckTOR10}, and probabilistic inference under the smProbLog semantics~\citep{totis_de_raedt_kimmig_2023}.
More formally, let $X_{in}$ and $X_{out}$ be a partition of the variables in a propositional theory $\Pi$.
Consider two commutative semirings $\mathcal{R}_{in} = (R^{i},\oplus^i,\otimes^i,e_{\oplus}^i,e_{\otimes}^i)$ and $\mathcal{R}_{out} = (R^{o},\oplus^o,\otimes^o,e_{\oplus}^o,e_{\otimes}^o)$, connected by a transformation function $f : \mathcal{R}^{i} \to \mathcal{R}^{o}$ and two weight functions, $w_{in}$ and $w_{out}$, that associate each literal to a weight (i.e., a real number).
2AMC requires computing:
\begin{align}
  \begin{split}  
    2AMC(A) =& 
    \bigoplus\nolimits_{I_{out} \in \sigma(X_{out})}^{o} 
    \bigotimes\nolimits^{o}_{a \in I_{out}} 
    w_{out}(a) 
    \otimes^{o} \\ 
    & f(
      \bigoplus\nolimits_{I_{in} \in \delta(\Pi \mid I_{out})}^{i} \bigotimes\nolimits^{i}_{b \in I_{in}} w_{in}(b)  
      )  
  \end{split}
\end{align}
where $\sigma(X_{out})$ are the set of possible assignments to $X_{out}$ and $\delta(\Pi \mid I_{out})$ are the set of possible assignments to the variables of $\Pi$ that satisfy $I_{out}$.
We can cast inference under the credal semantics as a 2AMC problem~\citep{AzzRig2023-AIXIA-IC} by considering as inner semiring $\mathcal{R}_{in} = (\mathbb{N}^2, +, \cdot, (0,0), (1,1))$, where $+$ and $\cdot$ are component-wise and $w_{in}$ is a function associating $not\ q$ to $(0,1)$ and all the other literals to $(1,1)$, where $q$ is the query.
The first component $n_u$ of the elements $(n_u,n_l)$ of the semiring counts the models where the query is true while the second component $n_l$ counts all the models. 
The transformation function $f(n_u,n_l)$ returns a pair of values $(f_u,f_l)$ such that $f_l = 1$ if $n_l = n_u$, 0 otherwise, and $f_u = 1$ if $n_u > 0$, 0 otherwise.
The outer semiring $\mathcal{R}_{out} = ([0, 1]^2, +, \cdot, (0,0), (1,1))$ is a double probability semiring, where there is a separate probability semiring for each component.
$w_{out}$ associates $a$ to $(p,p)$ and $not\ a$ to $(1-p,1-p)$ for every probabilistic fact $p :: a$, while it associates all the remaining literals to $(1,1)$.
aspmc, and so aspcs, solves 2AMC using knowledge compilation~\citep{DBLP:journals/jair/DarwicheM02} targeting NNF circuits~\citep{DBLP:conf/ecai/Darwiche04} where the order of the variables is guided by the treewidth of the program~\citep{DBLP:conf/kr/EiterHK21}.
Differently from PASTA, aspcs does not support aggregates, but we can convert rules and constraints containing them into new rules and constraints without aggregates using standard techniques~\citep{brewka2011asp}.



\subsection{Approximate Inference}
\label{subsec:approximate_inference}
Approximate inference can be performed by using the definition of the sampling semantics and returning the results after a finite number of samples, similarly to what is done for programs under the DS~\citep{DBLP:journals/tplp/KimmigDRCR11,Rig13-FI-IJ}.
It can be performed both on the discretized program and directly on the hybrid program.

\paragraph{\revisione{\textbf{Sampling the discretized program}}.}
In the discretized program, we can speed up the computation by storing the sampled worlds to avoid calling again the answer set solver in case a world has been already sampled.
However, the number of discrete probabilistic facts obtained via the conversion is heavily dependent on the types of constraints in the program.

\paragraph{\revisione{\textbf{Sampling the hybrid program}}.}
Sampling the hybrid program has the advantage that it allows general numerical constraints, provided they involve only continuous random variables and constants.
In this way, we directly sample the continuous random variables and directly test them against the constraints that can be composed of multiple variables and complex expressions.
In fact, when constraints among random variables are considered, it is difficult to discretize the domain.
Even if the samples would be always different (since the values are floating point numbers), also here we can perform caching, as in the discretized case: the constraints, when evaluated, are still associated with a Boolean value (true or false).
So, we can store the values of the evaluations of each constraint: if a particular configuration has already been encountered, we retrieve its contribution to the probability, rather than calling the ASP solver.

\paragraph{\revisione{\textbf{Approximate algorithm description.}}}
\revisione{Algorithm~\ref{alg:sampling} illustrates the pseudocode for the sampling procedure: for the sampling on the discretized program, the algorithm discretizes the program by calling \textsc{Discretize} (Algorithm~\ref{alg:discretize}).
Then, for a number of samples $s$, it samples a world by including or not every probabilistic fact into the program (function \textsc{SampleWorld}) according to the probability values, computes its answer sets with function \textsc{ComputeAnswerSets} (recall that a world is an answer set program), checks whether the query $q$ is true in every answer set (function \textsc{QueryInEveryAnswerSet}) or in at least one answer set (function \textsc{QueryInAtLeastOneAnswerSet}), and updates the lower and upper probability bounds accordingly.
At the end of the $s$ iterations, it returns the ratio between the number of samples contributing to the lower and upper probability and the number of samples taken.
The procedure is analogous (but without discretization) in the case the sampling on the original program.
A world is sampled with function \textsc{SampleVariablesAndTestConstraints}: it takes a sample for each continuous random variable, tests that value against each constraint specified in the program, and removes it if the test succeeds, otherwise it removes the whole rule containing it.
The remaining part of the algorithm is the same.    
}
\begin{algorithm}[t]
\begin{scriptsize}
\caption{\revisione{Function \textsc{Sampling}: computation of the probability of a query $q$ by taking $s$ samples in a hybrid probabilistic answer set program $P$ with rules $R$, continuous probabilistic facts $C$, and discrete probabilistic facts $D$.
Variable $type$ indicates whether the sampling of the discretized program or the original program is selected.}}
\label{alg:sampling}
\begin{algorithmic}[1]

\Function{Sampling}{$R,C,D,q,s,type$}
  \If {$type == discretized$}
    \State $P^s \gets$ \Call{Discretize}{$R,C,D$}
  \Else
    \State $P^s \gets P$
  \EndIf
  
  \State $i \gets 0$

  \While{$i < s$}
    \If {$type == discretized$}
      \State $w \gets$ \Call{SampleWorld}{$P^s$}
    \Else
      \State $w \gets$ \Call{SampleVariablesAndTestConstraints}{$P^s$}
    \EndIf
    \State $as \gets$ \Call{ComputeAnswerSets}{$w$}
    \If {\Call{QueryInEveryAnswerSet}{$as,q$}}
      \State $lb = lb + 1$
      \State $ub = ub + 1$
    \ElsIf {\Call{QueryInAtLeastOneAnswerSet}{$as,q$}}
      \State $ub = ub + 1$
    \EndIf
    \State $i \gets i + 1$
  \EndWhile
  \State \Return $lp/s, up/s$
\EndFunction

\end{algorithmic}
\end{scriptsize}
\end{algorithm}
  
We now present two results regarding the complexity of the sampling algorithm.
The first provides a bound on the number of samples needed to obtain an estimate of
 the upper (or lower) probability of a query within a certain absolute error.
The second result provides a bound on the number of samples needed to obtain 
an estimate within a certain relative error.
\begin{theorem}[Absolute Error]
\label{thm:complexity_sampling_absolute}
\revisione{
Let $q$ be a query in a hybrid probabilistic answer set program $P$ whose exact lower (upper) probability of success is $p$.
Suppose the sampling algorithm takes $s$ samples, $k$ of which are successful, and returns 
an estimate $\hat{p} = \frac{k}{s}$ of the lower (upper) probability of success. 
Let $\epsilon$ and $\delta$ be two numbers in $[0,1]$.
Then, the probability that $\hat{p}$ is within $\epsilon$ of $p$ is at least $1-\delta$, i.e.,
$$P(p-\epsilon \leq \hat{p} \leq p+\epsilon) \geq 1 - \delta$$
if 
$$
s\geq \frac{\epsilon+\frac{1}{2}}{\epsilon^2 \delta}.
$$
}
\end{theorem}
\begin{proof}
\revisione{
We must prove that
\begin{equation}
P(p-\epsilon \leq \hat{p}\leq p+\epsilon)\geq 1 - \delta\label{sampling-ineq}
\end{equation}
or, equivalently, that
\begin{equation}
P(sp-s\epsilon \leq k\leq sp+s\epsilon) \geq 1 - \delta.\label{sampling-ineq-count}
\end{equation}
}
\revisione{
Since $k$ is a binomially distributed random variable with number of trials $s$ and probability of success $p$, we have that~\citep{feller1968introduction}:
\begin{equation}
P(k\geq r_1) \leq \frac{r_1(1-p)}{(r_1-sp)^2}\label{upper-interval}
\end{equation}
if $r_1\geq sp$.
Moreover
\begin{equation}
\label{lower-interval}
P(k\leq r_2) \leq \frac{(s-r_2)p}{(sp-r_2)^2}
\end{equation}
if $r_2\leq sp$.
Since $P(k \geq r_1) = 1-P(k < r_1)$, from Equation~\ref{upper-interval} we have
\begin{eqnarray}
1-P(k<r_1)&\leq& \frac{r_1(1-p)}{(r_1-sp)^2}\nonumber\\
P(k<r_1)&\geq& 1-\frac{r_1(1-p)}{(r_1-sp)^2}\label{lb-lower-interval}
\end{eqnarray}
if $r_1\geq sp$.
}

\revisione{
In our case, we have $r_1 = sp+s\epsilon$ (since $r_1 \geq sp$) and $r_2 = sp-s\epsilon$ (since $r_2 \leq sp$).
So
\begin{eqnarray*}
&& P(p-\epsilon \leq \hat{p}\leq p+\epsilon) =\\
& = & P(r_2 \leq k \leq r_1) \\
& = &P(k\leq r_1)-P(k< r_2) \\
&\geq&1-\frac{r_1(1-p)}{(r_1-sp)^2}-\frac{(s-r_2)p}{(sp-r_2)^2} \qquad \text{(Eq.~\ref{lb-lower-interval} and~\ref{lower-interval} and since $P(x \leq v) \geq P(x < v)$)} \\
&=&1-\frac{(sp+s\epsilon)(1-p)}{(s\epsilon)^2}-\frac{(s-sp+s\epsilon)p}{(s\epsilon)^2} \quad \text{(by replacing the values of $r_1$ and $r_2$)} \\
&=&\frac{s^2\epsilon^2-sp-s\epsilon+sp^2+sp\epsilon-sp+sp^2-sp\epsilon}{s^2\epsilon^2}  \quad \text{(by  expanding)}   \\
&=&\frac{s\epsilon^2-2p-\epsilon+2p^2}{s\epsilon^2} \quad \text{(by collecting $s$ and simplifying)}\\
\end{eqnarray*}
and
\begin{eqnarray}
\frac{s\epsilon^2-2p-\epsilon+2p^2}{s\epsilon^2} & \geq & 1-\delta\nonumber\\
s\epsilon^2-2p-\epsilon+2p^2 & \geq & s\epsilon^2- s\epsilon^2 \delta\nonumber\\
s\epsilon^2\delta & \geq & 2p + \epsilon - 2p^2\nonumber\\
s & \geq & \frac{2p + \epsilon - 2p^2}{\epsilon^2\delta} = \frac{2p(1-p) + \epsilon}{\epsilon^2\delta}.\label{final-ineq} 
\end{eqnarray}
Since $0\leq 2p (1-p) \leq \frac{1}{2}$ with $p \in [0,1]$, then Equation~\ref{final-ineq} is implied by
\begin{eqnarray*}
s & \geq & \frac{\epsilon+\frac{1}{2}}{\epsilon^2 \delta}
\end{eqnarray*}
}
\end{proof}

\begin{theorem}[Relative Error]
\label{thm:complexity_sampling_relative}
\revisione{
Let $q$ be a query in a hybrid probabilistic answer set program $P$ whose exact lower (upper) probability of success is $p$.
Suppose the sampling algorithm takes $s$ samples, $k$ of which are successful,
and returns an estimate $\hat{p} = \frac{k}{s}$ of the lower (upper) probability of success. 
Let $\epsilon$ and $\delta$ be two numbers in $[0,1]$.
Then, the probability that the error between $\hat{p}$  and $p$ is smaller than $p\epsilon$ is at least $1-\delta^p$, i.e.,
$$P(|p-\hat{p}| \leq p\epsilon) \geq 1 - \delta^p$$
if 
$$
s\geq \frac{3}{\epsilon^2}\ln(\frac{1}{\delta}).
$$
}
\end{theorem}
\begin{proof}
\revisione{
According to Chernoff's bound~\citep[Corollary 4.6]{mitzenmacher2017probability}
we have that
$$
P( |ps-k| \ge ps\epsilon) \leq 2e^{-\frac{\epsilon^2ps}{3}}.
$$
So
$$
P( |p - \hat{p}| \geq p\epsilon) \le 2e^{-\frac{\epsilon^2ps}{3}}
$$
and
$$
P( |p - \hat{p}| \le p\epsilon) \geq 1-2e^{-\frac{\epsilon^2ps}{3}}.
$$
Then,
\begin{eqnarray*}
1-2e^{-\frac{\epsilon^2ps}{3}} & \geq & 1-\delta^p\\
2e^{-\frac{\epsilon^2ps}{3}} & \leq & \delta^p\\
\ln(2) -\frac{\epsilon^2ps}{3} & \leq & p \ln(\delta) \\
-\frac{\epsilon^2ps}{3} & \leq & p \ln(\delta) - \ln(2) \leq p \ln(\delta) \\
s & \geq & \frac{3}{\epsilon^2}\ln(\frac{1}{\delta}).
\end{eqnarray*}
}
\end{proof}
\revisione{
Please note that in Theorem \ref{thm:complexity_sampling_relative} we exponentiate $\delta$ to $p$.
This is needed to avoid the appearance of $p$ in  the bound, since $p$ is unknown.
When $p$ is 0, we have no guarantees on the error.
When $p$ is 1, the confidence is $1-\delta$.
For $p$ growing from 0 to 1, the bound provides increasing confidence.
In fact, approximating small probabilities is difficult due to the low success probability of the sample.
}

\section{Experiments}
\label{sec:experiments}
We ran experiments on a computer with 8 GB of RAM and a time limit of 8 hours (28800 seconds).
We generated five synthetic datasets, $t1$, $t2$, $t3$, $t4$, and $t5$, where every world of all the discretized versions of the programs has at least one answer set.
We use the SciPy library~\citep{2020scipy} to sample continuous random variables.
The following code snippets show the PASTA-backed programs; the only difference with the ones for aspcs is in the negation symbol, $not$ for the former and $\backslash+$ for the latter.
For all the datasets, we compare the exact algorithms and the approximate algorithms based on sampling, with an increasing number of samples.
We record the time required to parse the program and to convert the HPASP into a PASP together with the inference time.
The time for the first two tasks is negligible with respect to the inference time.

\paragraph{\revisione{\textbf{Dataset t1}}.}
In $t1$ we consider instances of increasing size of the program shown in Example~\ref{ex:example_interval}.
Every instance of size $n$ has 
$n/2$ discrete probabilistic facts $d_i$, 
$n/2$ continuous random variables $c_i$ with a Gaussian distribution with mean 0 and variance 1, 
$n/2$ pair of rules $q0 \impl below(c_i, 0.5), not \ q1$ and $q1 \impl below(c_i, 0.5), not \ q0$, $i = \{1,\dots,n/2\}$, and 
$n/2$ rules $q0 \impl below(c_i, 0.7), d_i$, one for each $i$ for $i=\{1,\dots,n/2\}$.
The query is $q0$.
The goal of this experiment is to analyze the behaviour of the algorithm by increasing the number of discrete and continuous random variables.
The instance of size 2 is:
\begin{lstlisting}
0.5::d1.
c1:gaussian(0,1).
q0 :- below(c1,0.5), not q1.
q1 :- below(c1,0.5), not q0.
q0 :- below(c1,0.7), d1.
\end{lstlisting}

\paragraph{\revisione{\textbf{Dataset t2}}.}
In $t2$ we consider a variation of $t1$, where we fix the number $k$ of discrete probabilistic facts to 2, 5, 8, and 10, and increase the number of continuous random variables starting from 1. 
Every instance of size $n$ contains 
$k$ discrete probabilistic facts $d_i$, $i \in \{0,\dots,k-1\}$,
$n$ continuous random variables $c_i$, $i \in \{1,\dots,n\}$ with a Gaussian distribution with mean 0 and variance 1, 
$n$ pair of rules $q0 \impl below(c_i, 0.5), not \ q1$ and $q1 \impl below(c_i, 0.5), not \ q0$, $i = \{1,\dots,n\}$, and 
$n$ rules $q0 \impl below(c_i, 0.7), d_{(i-1) \% k}$, one for each $i$ for $i=\{1,\dots,n\}$.
The query is $q0$.
Here, when the instance size $n$ is less than the number of discrete probabilistic facts $k$, $k - n$ probabilistic facts are not relevant for the computation of the probability of the query. 
The instance of size 2 with $k = 2$ is:
\begin{lstlisting}
0.5::d0.
0.5::d1.
c1:gaussian(0,1).
c2:gaussian(0,1).
q0 :- below(c1,0.5), not q1.
q1 :- below(c1,0.5), not q0.
q0 :- below(c2,0.5), not q1.
q1 :- below(c2,0.5), not q0.
q0 :- below(c1,0.7), d0.
q0 :- below(c2,0.7), d1.
\end{lstlisting}

\paragraph{\revisione{\textbf{Dataset t3}}.}
In $t3$, we consider again a variation of $t1$ where we fix the number $k$ of continuous random variables to 2, 5, 8, and 10, and increase the number of discrete probabilistic facts starting from 1. 
Every instance of size $n$ contains 
$k$ continuous probabilistic facts $c_i$, $i \in \{0,\dots,k-1\}$, with a Gaussian distribution with mean 0 and variance 1, 
$n$ discrete probabilistic facts $d_i$, $i \in \{1,\dots,n\}$,
$k$ pair of rules $q0 \impl below(c_i, 0.5), not \ q1$ and $q1 \impl below(c_i, 0.5), not \ q0$, $i = \{1,\dots,n\}$, and 
$n$ rules $q0 \impl below(c_{(i-1) \% k}, 0.7), d_i$, one for each $i$ for $i=\{1,\dots,n\}$.
The query is $q0$.
The instance of size 3 with $k = 2$ is:
\begin{lstlisting}
0.5::d1.
0.5::d2.
0.5::d3.
c0:gaussian(0,1).
c1:gaussian(0,1).
q0 :- below(c0,0.5), not q1.
q1 :- below(c0,0.5), not q0.
q0 :- below(c1,0.5), not q1.
q1 :- below(c1,0.5), not q0.
q0 :- below(c0,0.7), d1.
q0 :- below(c1,0.7), d2.
q0 :- below(c0,0.7), d3.
\end{lstlisting} 

\paragraph{\revisione{\textbf{Dataset t4}}.}
In $t4$ we consider programs with one discrete probabilistic fact $d$ and one continuous random variable $a$ that follows a Gaussian distribution with mean 0 and standard deviation 10.
An instance of size $n$ has 
$n$ pairs of rules $q0 \impl between(a,lb_i,ub_i), not \ q1$ and $q1 \impl between(a,lb_i,ub_i), not \ q0$ where $lb_i$ and $ub_i$ are randomly generated (with, $lb_i < ub_i$), for $j = 1,\dots,n$, and 
$n$ rules of the form $q0 \impl d, between(a,LB_j,UB_j)$, where the generation of the $LB_j$ and $UB_j$ follows the same process of the previous rule.
We set the minimum value of $lb_i$ and $LB_j$ to -30 and, for both rules, the lower and upper bounds for the $\mathit{between}/2$ comparison predicate are uniformly sampled in the range $[u_{i-1}, u_{i-1} + 60/n]$, where $u_{i-1}$ is the upper bound of the previous rule.
The query is $q0$.
Here, the goal is to test the algorithm with an increasing number of intervals to consider.
An example of instance of size 2 is:
\begin{lstlisting}
0.4::d.
c:gaussian(0,10).
q0:- between(c,-30,-23.606), not q1.
q1:- between(c,-30,-23.606), not q0.
q0:- d, between(c,-30,-29.75).
\end{lstlisting}

\paragraph{\revisione{\textbf{Dataset t5}}.}
Lastly, in $t5$ we consider programs of the form shown in Example~\ref{ex:blood_pressure}: the instance of index $n$ has $n$ people involved, $n$ discrete probabilistic facts, and $n$ $d/1$ and $s/1$ continuous random variables.
The remaining part of the program is the same.
Example~\ref{ex:blood_pressure} shows the instance of index 4 (4 people involved).
The query is $high\_number\_strokes$.

\begin{figure}[tb]
\centering
\begin{subfigure}{0.48\textwidth}
\centering
\resizebox{\resizeGraphFactor\textwidth}{!}{%
\begin{tikzpicture}
\begin{axis}[
    xlabel={Instance},
    ylabel={Execution Time (s)},
    legend pos=north east,
    grid style=dashed,
    legend cell align={left},every axis plot/.append style={ultra thick}
    ]

\addplot[color = first_plot, dotted]
    coordinates {
      (2,4.11)
      (4,3.985)
      (6,4.0)
      (8,4.012)
      (10,4.117)
      (12,4.588)
      (14,5.722)
      (16,10.078)
      (18,43.779)
      (20,61.043)
      (22,111.646)
      (24,267.987)
      (26,490.21)
    };
\addlegendentry{t1 aspcs}

\addplot[color = second_plot, loosely dotted] 
    coordinates {
      (2, 0.762)
      (4, 0.557)
      (6, 0.566)
      (8, 0.934)
      (10, 4.56)
      (12, 36.377)
      (14, 315.14)
      (16, 3211.299)
    };
\addlegendentry{t1 PASTA}

\addplot[color = third_plot, densely dotted] 
    coordinates {
      (21,66.93)
      (22,61.512)
      (23,61.474)
      (24,64.023)
      (25,259.358)
      (26,89.681)
      (27,114.974)
      (28,98.145)
      (29,104.144)
      (30,161.094)
      (31,105.274)
      (32,134.64)
      (33,105.199)
      (34,155.023)
      (35,2698.13)
    };
\addlegendentry{t4 aspcs}

\addplot[color = fourth_plot, dashed] 
    coordinates {
      (2,1.253)
      (3,1.175)
      (4,1.225)
      (5,1.655)
      (6,3.462)
      (7,11.808)
      (8,46.55)
      (9,196.455)
      (10,843.472)
      (11,3683.347)
    };
\addlegendentry{t4 PASTA}

\addplot[color = fifth_plot, loosely dashed] 
    coordinates {
      (1,4)
      (2,4.061)
      (3,4.01)
      (4,4)
      (5,4)
      (6,4)
      (7,4)
      (8,4)
      (9,16)
    };
\addlegendentry{t5 aspcs}

\addplot[color = sixth_plot, dashed, densely dashed] 
    coordinates {
      (1, 1)
      (2, 1)
      (3, 61)
      (4, 5392)
    };
\addlegendentry{t5 PASTA}

\end{axis}
\end{tikzpicture}
}
\caption{$t1$, $t4$, and $t5$.}
\label{fig:t1_t4_t5_results}
\end{subfigure}%
\hfill
\begin{subfigure}{0.48\textwidth}
\resizebox{\resizeGraphFactor\textwidth}{!}{%
\begin{tikzpicture}
\begin{axis}[
    xlabel={Number of Discrete Facts},
    ylabel={Execution Time (s)},
    legend pos=north east,
    grid style=dashed,
    legend cell align={left},every axis plot/.append style={ultra thick}
    ]

    \addplot[color = first_plot] 
    coordinates {
      (20,12.241)
      (25,4.785)
      (30,4.806)
      (35,5.052)
      (40,6.208)
      (45,4.995)
      (50,5.404)
      (55,6.775)
      (60,5.401)
      (65,12.702)
      (70,5.414)
      (75,6.706)
      (80,6.007)
      (85,11.529)
      (90,53.435)
      (95,6.022)
      (100,26.732)
      (110,347.963)
    };
\addlegendentry{2}

\addplot[color = second_plot] 
    coordinates {
      (20,10.005)
      (25,5.389)
      (30,7.896)
      (35,7.711)
      (40,7.157)
      (45,16.518)
      (50,6.934)
      (55,11.332)
      (60,38.808)
      (65,17.628)
      (70,8.343)
      (75,39.906)
    };

\addlegendentry{5}

\end{axis}
\end{tikzpicture}
}
\caption{2 and 5 continuous facts.}
\label{fig:t3_2_5_big}
\end{subfigure}
\caption{Results for $t1$, $t4$, and $t5$ (left) and results for aspcs applied to $t3$ (right) with 2 and 5 continuous facts.}
\label{fig:results_t1_t4_t5_t3_big}
\end{figure}
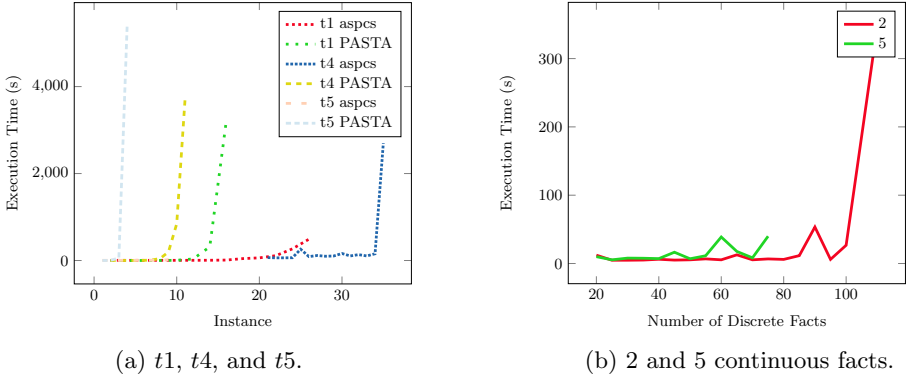

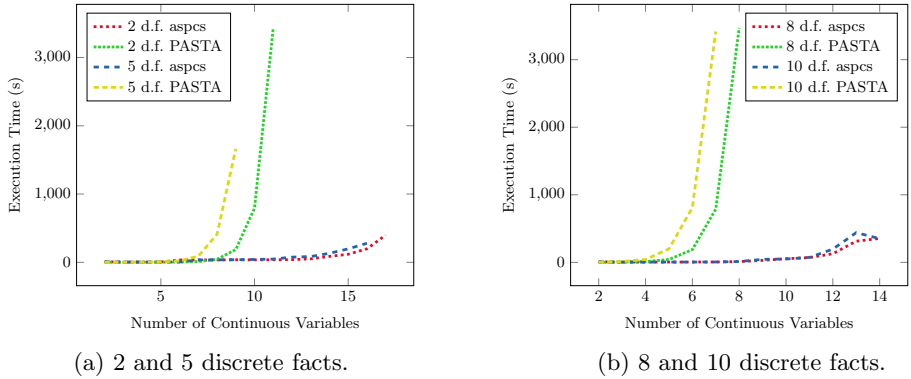
\begin{figure}[tb]
\centering
\begin{subfigure}{0.48\textwidth}
\centering
\resizebox{\resizeGraphFactor\textwidth}{!}{%
\begin{tikzpicture}
\begin{axis}[
    xlabel={Number of Continuous Variables},
    ylabel={Execution Time (s)},
    legend pos=north west,
    grid style=dashed,
    legend cell align={left},every axis plot/.append style={ultra thick}
    ]


\addplot[color = first_plot, dotted] 
    coordinates {
      (2,7.323)
      (3,4.153)
      (4,4.188)
      (5,5.169)
      (6,34.229)
      (7,34.437)
      (8,34.598)
      (9,37.614)
      (10,36.012)
      (11,36.981)
      (12,39.548)
      (13,50.463)
      (14,86.703)
      (15,117.987)
      (16,197.576)
      (17,397.857)
    };
\addlegendentry{2 d.f. aspcs}

\addplot[color = second_plot, densely dotted] 
    coordinates {
      (2,2.265)
      (3,1.159)
      (4,1.186)
      (5,2.822)
      (6,3.179)
      (7,10.688)
      (8,44.187)
      (9,186.835)
      (10,792.863)
      (11,3434.921)
    };
\addlegendentry{2 d.f. PASTA}

\addplot[color = third_plot, dashed] 
    coordinates {
      (2,7.09)
      (3,4.157)
      (4,4.156)
      (5,4.329)
      (6,4.65)
      (7,34.995)
      (8,36.214)
      (9,37.696)
      (10,41.742)
      (11,47.419)
      (12,75.924)
      (13,84.244)
      (14,132.102)
      (15,197.18)
      (16,278.845)
    };
\addlegendentry{5 d.f. aspcs}

\addplot[color = fourth_plot, densely dashed] 
    coordinates {
      (2,1.469)
      (3,1.369)
      (4,2.109)
      (5,5.442)
      (6,20.558)
      (7,90.061)
      (8,411.522)
      (9,1663.285)
    };
\addlegendentry{5 d.f. PASTA}

\end{axis}
\end{tikzpicture}
}
\caption{2 and 5 discrete facts.}
\label{subfig:t2_2_5_results}
\end{subfigure}%
\hfill
\begin{subfigure}{0.48\textwidth}
\resizebox{\resizeGraphFactor\textwidth}{!}{%
\begin{tikzpicture}
\begin{axis}[
    xlabel={Number of Continuous Variables},
    ylabel={Execution Time (s)},
    legend pos=north east,
    grid style=dashed,
    legend cell align={left},every axis plot/.append style={ultra thick}
    ]

\addplot[color = first_plot, dotted] 
coordinates {
  (2,7.473)
  (3,4.159)
  (4,4.15)
  (5,4.324)
  (6,4.72)
  (7,5.797)
  (8,11.239)
  (9,29.568)
  (10,53.995)
  (11,71.074)
  (12,127.587)
  (13,313.236)
  (14,354.532)
};
\addlegendentry{8 d.f. aspcs}

\addplot[color = second_plot, densely dotted] 
coordinates {
  (2,1.88)
  (3,3.431)
  (4,10.372)
  (5,42.699)
  (6,188.228)
  (7,788.135)
  (8,3471.223)
};
\addlegendentry{8 d.f. PASTA}

\addplot[color = third_plot, dashed] 
coordinates {
  (2,6.929)
  (3,4.16)
  (4,4.154)
  (5,4.338)
  (6,4.742)
  (7,5.684)
  (8,11.896)
  (9,44.803)
  (10,49.825)
  (11,73.189)
  (12,195.55)
  (13,440.046)
  (14,350.53)
};
\addlegendentry{10 d.f. aspcs}

\addplot[color = fourth_plot, densely dashed] 
coordinates {
  (2,3.988)
  (3,11.179)
  (4,43.378)
  (5,199.38400000000001)
  (6,808.621)
  (7,3416.981)
};
\addlegendentry{10 d.f. PASTA}

\end{axis}
\end{tikzpicture}
}
\caption{8 and 10 discrete facts.}
\label{subfig:t2_8_10}
\end{subfigure}
\caption{Results for the experiment $t2$ with a fixed number of discrete facts and an increasing number of continuous variables.}
\label{fig:t2}
\end{figure}

\begin{figure}[tb]
\centering
\begin{subfigure}{0.48\textwidth}
\centering
\resizebox{\resizeGraphFactor\textwidth}{!}{%
\begin{tikzpicture}
\begin{axis}[
    xlabel={Number of Discrete Facts},
    ylabel={Execution Time (s)},
    legend pos=north west,
    grid style=dashed,
    legend cell align={left},every axis plot/.append style={ultra thick}
    ]

\addplot[color = first_plot, dotted] 
    coordinates {
      (2,4.157)
      (3,4.127)
      (4,4.115)
      (5,4.07)
      (6,4.092)
      (7,4.122)
      (8,4.181)
      (9,4.139)
      (10,4.116)
      (11,4.132)
      (12,4.19)
      (13,4.128)
      (14,4.159)
      (15,4.154)
      (16,4.192)
      (17,4.207)
      (18,4.243)
      (19,4.242)
      (20,4.337)
    };
\addlegendentry{2 c.f. aspcs}

\addplot[color = second_plot, densely dotted] 
    coordinates {
      (2,1.39)
      (3,1.208)
      (4,1.2)
      (5,1.226)
      (6,1.278)
      (7,1.39)
      (8,1.684)
      (9,2.212)
      (10,3.189)
      (11,5.619)
      (12,10.431)
      (13,20.343)
      (14,42.927)
      (15,85.606)
      (16,176.457)
      (17,370.208)
      (18,746.456)
      (19,1538.047)
      (20,3282.658)
      };
\addlegendentry{2 c.f. PASTA}

\addplot[color = third_plot, dashed] 
    coordinates {
      (2,4.217)
      (3,4.173)
      (4,4.352)
      (5,4.245)
      (6,4.37)
      (7,4.335)
      (8,4.392)
      (9,4.35)
      (10,4.449)
      (11,4.453)
      (12,4.466)
      (13,4.597)
      (14,4.651)
      (15,4.516)
      (16,4.753)
      (17,4.609)
      (18,4.618)
      (19,4.525)
      (20,4.538)
    };
\addlegendentry{5 c.f. aspcs}

\addplot[color = fourth_plot, densely dashed] 
    coordinates {
      (2,1.324)
      (3,1.472)
      (4,2.305)
      (5,5.77)
      (6,10.948)
      (7,21.229)
      (8,43.38)
      (9,88.054)
      (10,183.878)
      (11,388.751)
      (12,775.672)
      (13,1618.839)
      (14,3413.759)
    };
\addlegendentry{5 c.f. PASTA}

\end{axis}
\end{tikzpicture}
}
\caption{2 and 5 continuous facts.}
\label{subfig:t3_2_5_results}
\end{subfigure}%
\hfill
\begin{subfigure}{0.48\textwidth}
\resizebox{\resizeGraphFactor\textwidth}{!}{%
\begin{tikzpicture}
\begin{axis}[
    xlabel={Number of Discrete Facts},
    ylabel={Execution Time (s)},
    legend pos=north east,
    grid style=dashed,
    legend cell align={left},every axis plot/.append style={ultra thick}
    ]

\addplot[color = first_plot, dotted] 
coordinates {
  (2,5.003)
  (3,5.07)
  (4,5.428)
  (5,6.101)
  (6,7.121)
  (7,8.996)
  (8,11.611)
  (9,25.872)
  (10,25.135)
  (11,26.365)
  (12,19.387)
  (13,15.832)
  (14,34.154)
  (15,23.141)
  (16,34.916)
  (17,31.168)
  (18,34.034)
  (19,27.113)
  (20,18.496)
};
\addlegendentry{8 c.f. aspcs}

\addplot[color = second_plot, densely dotted] 
coordinates {
  (2,1.958)
  (3,3.738)
  (4,11.611)
  (5,45.283)
  (6,191.951)
  (7,799.413)
  (8,3350.462)
};
\addlegendentry{8 c.f. PASTA}

\addplot[color = third_plot, dashed] 
coordinates {
  (2,9.72)
  (3,13.308)
  (4,18.212)
  (5,22.63)
  (6,31.954)
  (7,41.59)
  (8,45.034)
  (9,48.312)
  (10,62.498)
  (11,66.366)
  (12,76.385)
  (13,86.417)
  (14,94.203)
  (15,108.922)
  (16,129.706)
  (17,151.718)
  (18,116.179)
  (19,139.287)
  (20,314.925)
};
\addlegendentry{10 c.f. aspcs}

\addplot[color = fourth_plot, densely dashed] 
coordinates {
  (2,10.346)
  (3,10.663)
  (4,42.651)
  (5,168.843)
  (6,708.601)
  (7,2981.6)
};
\addlegendentry{10 c.f. PASTA}

\end{axis}
\end{tikzpicture}
}
\caption{8 and 10 continuous facts.}
\label{subfig:t3_8_10}
\end{subfigure}
\caption{Results for the experiment $t3$ with a fixed number of continuous random variables and an increasing number of discrete facts. The x axis of the left plot is cut at size 20, to keep the results readable.}
\label{fig:t3}
\end{figure}
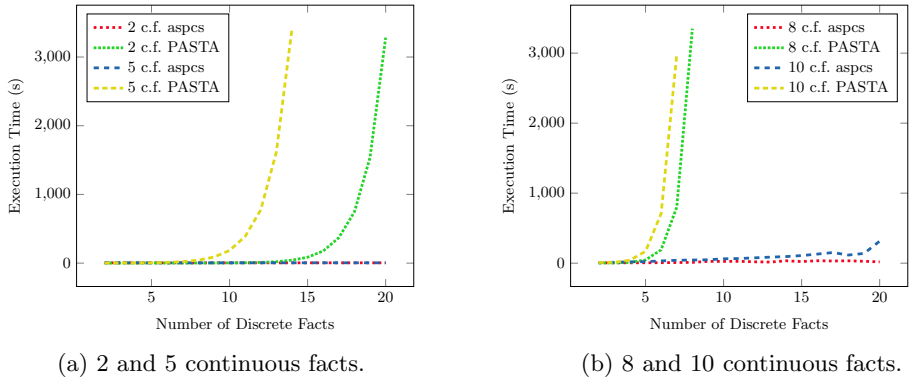

\begin{table}[tb] 
\caption{Results for the approximate algorithms based on sampling.
The columns contain, from the left, 
the name of the dataset,
the instance index,
and the average time required to take $10^2$, $10^3$, $10^4$, $10^5$, and $10^6$ samples on the original and converted program.
O.O.M. and T.O. stand respectively for out of memory and timeout.
}
\label{tab:results_sampling_means}
\centering
{\footnotesize
\begin{tabular}{| c | c | c | c | c | c | c |} 
data. & inst. & $10^2$ s. O./C. & $10^3$ s. O./C. & $10^4$ s. O./C. & $10^5$ s. O./C. & $10^6$ s. O./C. \\
\hline
\hline
t1 & 50 & 2.81/2.21 & 12.46/2.07 & 116.07/10.19 & 1200.19/90.83 & 11998.92/906.13 \\
t1 & 60 & 3.05/2.13 & 14.50/2.10 & 139.83/11.19 & 1417.36/103.63 & 14154.23/1059.09 \\
t1 & 70 & 3.80/2.16 & 17.05/2.25 & 167.39/12.69 & 1645.97/117.31 & 16523.49/1222.05 \\
t1 & 80 & 4.00/1.74 & 19.26/2.42 & 190.23/14.29 & 1859.69/135.09 & 18524.07/1375.32 \\
t1 & 90 & 4.26/1.89 & 21.44/2.63 & 215.25/16.03 & 2060.15/152.52 & 20518.89/1558.51 \\
t1 & 100 & 4.51/1.96 & 23.69/2.90 & 231.85/17.95 & 2254.84/170.81 & 22692.96/1751.96 \\
\hline
t2 & 50 & 3.77/1.35 & 22.80/2.82 & 223.95/17.21 & 2299.33/163.13 & 22277.69/1654.94 \\
t2 & 60 & 4.22/1.33 & 27.02/3.09 & 269.97/20.14 & 2752.39/194.99 & 26532.31/1986.42 \\
t2 & 70 & 4.63/1.42 & 31.49/3.43 & 312.66/23.54 & 3134.76/226.24 & T.O./2271.90 \\
t2 & 80 & 4.65/1.6 & 36.00/3.88 & 363.20/27.02 & 3616.07/261.18 & T.O./2646.70 \\
t2 & 90 & 5.42/1.66 & 42.78/4.24 & 418.04/30.26 & 4006.67/293.14 & T.O./2926.46 \\
t2 & 100 & 5.77/2.18 & 46.35/4.63 & 457.80/34.05 & 4520.69/327.29 & T.O./3286.04 \\
\hline
t3 & 50 & 3.90/1.82 & 25.46/3.18 & 245.32/19.08 & 2439.00/176.84 & 21818.69/1808.43 \\
t3 & 60 & 4.71/1.89 & 29.50/3.59 & 288.98/22.79 & 2817.13/215.89 & 26084.07/2176.02 \\
t3 & 70 & 5.17/1.86 & 34.16/3.83 & 335.37/25.34 & 3247.42/241.67 & T.O./2452.97 \\
t3 & 80 & 5.51/2.03 & 38.71/4.36 & 382.96/29.40 & 3685.38/281.14 & T.O./2823.65 \\
t3 & 90 & 5.67/2.09 & 44.20/4.71 & 432.43/32.85 & 4123.31/318.48 & T.O./3154.16 \\
t3 & 100 & 6.15/2.15 & 48.83/5.17 & 476.51/36.47 & 4535.84/357.23 & T.O./3542.78 \\
\hline
t4 & 50 & 2.37/22.55 & 2.10/23.99 & 7.27/34.63 & 59.20/95.59 & 571.27/428.92 \\
t4 & 60 & 2.42/43.08 & 2.29/44.18 & 7.81/61.99 & 62.36/178.80 & 602.12/842.64 \\
t4 & 70 & 2.41/O.O.M. & 2.35/O.O.M. & 8.28/O.O.M. & 66.26/O.O.M. & 636.83/O.O.M. \\
t4 & 80 & 2.47/O.O.M. & 2.61/O.O.M. & 8.83/O.O.M. & 69.69/O.O.M. & 668.62/O.O.M. \\
t4 & 90 & 2.03/O.O.M. & 2.89/O.O.M. & 9.79/O.O.M. & 75.85/O.O.M. & 717.39/O.O.M. \\
t4 & 100 & 2.52/O.O.M. & 3.09/O.O.M. & 10.11/O.O.M. & 77.70/O.O.M. & 736.59/O.O.M. \\
\hline
t5 & 50 & 6.14/2.29 & 44.20/6.72 & 440.02/69.44 & 4433.25/1065.66 & T.O./24376.32 \\
t5 & 60 & 7.03/3.16 & 53.01/8.30 & 533.38/87.33 & 5437.58/1575.33 & T.O./23889.43 \\
t5 & 70 & 8.16/3.14 & 62.37/10.44 & 637.10/125.86 & 6497.07/1874.05 & T.O./T.O. \\
t5 & 80 & 8.82/3.26 & 70.79/12.62 & 728.98/163.81 & 7610.59/2855.01 & T.O./T.O. \\
t5 & 90 & 9.41/3.34 & 80.89/14.43 & 845.18/209.66 & 8931.12/3879.62 & T.O./T.O. \\
t5 & 100 & 10.26/4.01 & 89.73/16.35 & 936.90/248.00 & 9984.17/4310.61 & T.O./T.O. \\
\hline
\hline\end{tabular}}
\end{table}

\begin{table}[tb] 
\caption{Standard deviations for the results listed in Table~\ref{tab:results_sampling_means}.
A dash denotes that there are no results for that particular instance due to either memory error or time limit.
}
\label{tab:results_sampling_stddevs}
\centering
{\footnotesize
\begin{tabular}{| c | c | c | c | c | c | c |} 
data. & inst. & $10^2$ s. O./C. & $10^3$ s. O./C. & $10^4$ s. O./C. & $10^5$ s. O./C. & $10^6$ s. O./C. \\
\hline
\hline
t1 & 50 & 0.97/1.34 & 0.70/0.15 & 6.71/0.71 & 71.46/5.84 & 920.82/54.93 \\
t1 & 60 & 0.99/1.19 & 0.76/0.11 & 9.10/0.65 & 95.47/5.71 & 1003.31/50.87 \\
t1 & 70 & 1.29/1.17 & 0.67/0.12 & 10.55/0.73 & 86.36/7.61 & 1141.45/51.24 \\
t1 & 80 & 1.24/0.92 & 1.01/0.10 & 9.52/0.65 & 71.22/7.98 & 1381.12/64.04 \\
t1 & 90 & 1.24/1.01 & 1.29/0.11 & 10.65/0.71 & 56.47/7.00 & 1525.12/120.35 \\
t1 & 100 & 1.20/0.96 & 1.37/0.24 & 12.62/0.92 & 59.29/10.79 & 1754.89/99.47 \\
\hline
t2 & 50 & 0.92/0.23 & 0.68/0.22 & 11.56/0.51 & 150.47/3.92 & 438.66/39.72 \\
t2 & 60 & 0.86/0.03 & 1.03/0.06 & 15.06/0.40 & 162.91/3.81 & 581.83/53.66 \\
t2 & 70 & 0.93/0.11 & 0.86/0.10 & 20.00/0.44 & 158.68/7.14 & -/68.62 \\
t2 & 80 & 0.17/0.19 & 0.82/0.20 & 22.24/1.22 & 202.95/9.73 & -/50.98 \\
t2 & 90 & 0.47/0.19 & 3.66/0.30 & 30.78/1.43 & 218.51/12.05 & -/53.71 \\
t2 & 100 & 0.37/0.91 & 3.03/0.34 & 22.13/2.67 & 159.32/13.80 & -/60.19 \\
\hline
t3 & 50 & 0.07/0.11 & 0.85/0.11 & 5.92/0.59 & 75.95/7.63 & 324.79/56.21 \\
t3 & 60 & 0.57/0.14 & 1.81/0.14 & 12.05/1.58 & 140.43/10.00 & 262.68/63.53 \\
t3 & 70 & 0.54/0.20 & 2.09/0.24 & 15.22/1.75 & 162.67/12.74 & -/88.86 \\
t3 & 80 & 0.56/0.10 & 2.31/0.36 & 15.50/2.30 & 183.78/9.30 & -/67.44 \\
t3 & 90 & 0.20/0.13 & 1.43/0.34 & 22.31/2.25 & 173.10/15.10 & -/100.41 \\
t3 & 100 & 0.26/0.12 & 1.89/0.39 & 22.30/2.61 & 158.54/8.47 & -/93.45 \\
\hline
t4 & 50 & 1.22/0.31 & 0.15/1.52 & 0.27/3.66 & 2.78/6.86 & 6.35/13.43 \\
t4 & 60 & 1.19/1.22 & 0.17/3.23 & 0.35/4.77 & 3.08/16.63 & 12.71/27.04 \\
t4 & 70 & 1.21/- & 0.15/- & 0.35/- & 4.22/- & 18.99/- \\
t4 & 80 & 1.19/- & 0.17/- & 0.43/- & 4.27/- & 25.56/- \\
t4 & 90 & 0.93/- & 0.15/- & 0.44/- & 3.42/- & 27.86/- \\
t4 & 100 & 1.10/- & 0.18/- & 0.66/- & 4.61/- & 23.38/- \\
\hline
t5 & 50 & 0/0.92 & 1.92/0.51 & 17.83/8.27 & 122.23/208.32 & -/3971.11 \\
t5 & 60 & 0.93/1.56 & 2.01/0.66 & 22.59/10.41 & 177.61/238.23 & -/7447.05 \\
t5 & 70 & 1.10/1.46 & 2.46/0.93 & 20.09/14.83 & 158.03/400.35 & -/- \\
t5 & 80 & 0.92/0.88 & 2.97/1.27 & 34.17/26.08 & 303.35/709.13 & -/- \\
t5 & 90 & 0.82/0.94 & 2.15/1.37 & 37.42/28.92 & 442.68/684.92 & -/- \\
t5 & 100 & 0.92/1.36 & 3.81/1.17 & 47.29/26.12 & 793.68/745.80 & -/- \\
\hline
\hline\end{tabular}}
\end{table}

\paragraph{\revisione{\textbf{Exact inference results}}.}
\revisione{The goal of benchmarking the exact algorithms is threefold: 
i) identifying how the number of continuous and discrete probabilistic facts influences the execution time, 
ii) assessing the impact on the execution time of an increasing number of intervals, and 
iii) comparing knowledge compilation with projected answer set enumeration.}
Figure~\ref{fig:t1_t4_t5_results} shows the results of exact inference for the algorithms backed by PASTA (dashed lines) and aspcs (straight lines) on $t1$, $t4$, and $t5$.
For $t5$ PASTA is able to solve up to instance 4 while aspcs can solve up to instance 9 (in a  few seconds).
Similarly for $t1$, where aspcs doubles the maximum sizes solvable by PASTA.
The difference is even more marked for $t4$: here, PASTA solves up to instance 11 while aspcs up to instance 35. 
Figures~\ref{fig:t2} and~~\ref{fig:t3} show the results for $t2$ and $t3$.
Here, PASTA cannot solve instances with around more than 20 probabilistic facts and continuous random variables combined.
The performance of aspcs are clearly superior, in particular for $t3$ with 2 and 5 continuous random variables (Figure~\ref{subfig:t3_2_5_results}).
Note that this is the only plot where we cut the x axis at instance 20, to keep the results readable.
To better investigate the behaviour of aspcs in this case, we keep increasing the number of discrete facts until we get a memory error, while the number of continuous variables is fixed to 2 and 5. 
Figure~\ref{fig:t3_2_5_big} shows the results: with 2 continuous variables, we can solve up to instance 110 while with 5 continuous variables up to 75.
In general, the results of exact inference are in accordance with the theoretical complexity results.
However, as also empirically shown by~\cite{AzzRig2023-AIXIA-IC}, knowledge compilation has a huge impact on the execution times.
Overall, as expected, PASTA is slower in all the tests, being based on (projected) answer set enumeration.
For all, both PASTA (exact algorithm) and aspcs stop for lack of memory.

\paragraph{\revisione{\textbf{Approximate inference results}}.}
\revisione{The goal of the experiments run with approximate algorithms is:
i) analyzing the impact on the execution time of the number of samples taken,
ii) comparing the approach based on sampling the original program against the one based on sampling the converted program in terms of execution times, and 
iii) assessing the memory requirements.}
Table~\ref{tab:results_sampling_means} shows the averages on 5 runs of the execution time of the approximate algorithm applied to both the discretized and original program with $10^2$, $10^3$, $10^4$, $10^5$, and $10^6$ samples, for each of the five datasets.
Standard deviations are listed in Table~\ref{tab:results_sampling_stddevs}.
For $t2$ and $t3$ we consider programs with the same number of continuous variables and discrete probabilistic facts.
In four of the five tests (all except $t4$), sampling the original program is slower than sampling the converted program, sometimes by a significant amount.
This is probably due to the fact that sampling a continuous distribution is slower than sampling a Boolean random variable.
For example, in instance 100 of $t3$ sampling the original program is over twelve times slower than sampling the converted program.
However, the discretization process increases the number of probabilistic facts and so the required memory: for $t4$, from the instance 70, taking even 100 samples requires more than 8 GB of RAM.
To better assess the memory requirements, we repeat the experiments with the approximate algorithms with 6, 4, 2, and 1 GB of maximum available memory.
For all, taking up to $10^5$ samples in the original program is feasible also with only 1 GB of memory.
The same considerations hold for sampling the converted program, except for $t4$.
Table~\ref{tab:reducing_memory} shows the largest solvable instances together with the number of probabilistic facts, rules (for the converted program), and samples for each instance: for example, with 1 GB of memory it is possible to take only up to $10^4$ samples in the instance of size 40.
For this dataset, the increasing number of $between/3$ predicates requires an increasing number of rules and probabilistic facts to be included in the program during the conversion, to properly handle all the intervals: the instance of size 70 has 142 probabilistic facts and more than 30000 rules, that make the generation of the answer sets very expensive and sampling it with only 8 GB of memory is not feasible.

\begin{table}[tb]
\caption{Largest solvable instances of $t4$ by sampling the converted program when reducing the available memory.
The columns contain, from the left, 
the maximum amount of memory,
the largest solvable instance together with the number of probabilistic facts (\# p.f.) and rules (\# rules) obtained via the conversion,
and the maximum number of samples that can be taken (max. \# samples).
Note that we increase the number of samples by starting from $10^2$ and iteratively multiplying the number by 10, up to $10^6$, so in the last column we report a range: this means that we get a memory error with the upper bound while we can take the number of samples in the lower bound.
Thus, the maximum values of samples lies in the specified range.
}
\label{tab:reducing_memory}
\centering
\begin{tabular}{| c | c | c | c | c |} 
memory & instance & \# p.f. & \# rules & max. \# samples \\
\hline
\hline
6 GB & 60 & 122 & 22495 & $>10^6$ \\
4 GB & 50 & 102 & 16151 & $[10^5,10^6]$ \\
2 GB & 50 & 102 & 16151 & $[10^3,10^4]$ \\
1 GB & 40 & 82 & 10554  & $[10^4,10^5]$ \\
\hline
\end{tabular}
\end{table}

\section{Related Work}
\label{sec:related}

Probabilistic logic programs with discrete and continuous random variables have been the subject of various works.
~\cite{DBLP:conf/ilp/GutmannJR10} proposed Hybrid ProbLog, an extension of ProbLog~\citep{DBLP:conf/ijcai/RaedtKT07} to support continuous distributions.
There are several differences with Hybrid ProbLog: first, Hybrid ProbLog focuses on PLP while our approach on PASP.
Thus, the syntax and semantics are different.
For the syntax, in PASP we can use rich constructs such as aggregates, that greatly increase the expressivity of the programs.
For the semantics, at high level, PLP requires that every world (i.e., combination of probabilistic facts) has exactly one model while PASP does not.
Another difference with Hybrid ProbLog is in the discretization process: Hybrid ProbLog discretizes the proofs of a program while we directly discretize the program.
Moreover, their inference algorithm is based on the construction of a compact representation of the program via Binary Decision Diagrams, while we use both ASP solvers with projective solutions and knowledge compilation targeting NNF.

Distributional Clauses~\citep{DBLP:journals/tplp/GutmannTKBR11} and Extended PRISM~\citep{TLP:8688161} are two other proposals to handle both discrete and continuous random variables.
The semantics of the former is based on a stochastic extension of the immediate consequence $T_p$ operator, while the latter considers the least model semantics of constraint logic programs~\citep{DBLP:journals/jlp/JaffarM94} and extends the PRISM framework~\citep{DBLP:conf/iclp/Sato95}.
\cite{DBLP:journals/ai/MichelsHLV15} introduced Probabilistic Constraint Logic Programming whose semantics is based on an extension of Sato's Distribution Semantics~\citep{DBLP:conf/iclp/Sato95}.
\cite{AzzRigLam21-AIJ-IJ} proposed a semantics for hybrid probabilistic logic programs that allows a denumerable number of random variables.
In general, all the previously discussed approaches only support normal clauses, do not adopt some of the ASP constructs, such as aggregates and constraints, and require that the worlds have a single model.
Some languages allow handling uncertainty in ASP, such as LPMLN~\citep{DBLP:conf/kr/LeeW16}, P-log~\citep{DBLP:journals/tplp/BaralGR09}, PrASP~\citep{DBLP:conf/iclp/NicklesM15}, and differentiable SAT/ASP~\citep{DBLP:conf/ilp/Nickles18a} but none of these consider continuous distributions.
PASOCS~\citep{tuckey2021PASOCS} is a system for performing inference in probabilistic answer set programs under the credal semantics, but it does not allow worlds without answer sets and continuous variables while plingo~\citep{plingo2022} consider the LPMLN, P-log, and ProbLog semantics (the relationship among these has been discussed in~\citep{DBLP:conf/aaai/LeeY17}).
The credal least undefined semantics~\citep{DBLP:conf/kr/RochaC22} and the smProbLog semantics~\citep{totis_de_raedt_kimmig_2023} handle unsatisfiable worlds, but by a considering three-valued semantics and do not allow continuous random variables.

There is a large body of work on inference in Probabilistic Programming (PP)~\citep{pfeiffer2016,DBLP:journals/corr/TranHSBMB17,DBLP:conf/cav/GehrMV16,vandemeent2021introduction} with both discrete and continuous random variables, with several available tools~\citep{bingham2018pyro,phan2019composable,tran2016edward}.
PLP and PASP adopt a declarative approach to describe a domain, so they are particularly suitable for describing relational domains.
Translating a PLP/PASP into PP is possible but would result in a much longer and less understandable program.

\section{Conclusions}
\label{sec:conclusions}
In this paper we propose \textit{Hybrid} Probabilistic Answer Set Programming, an extension of Probabilistic Answer Set Programming under the credal semantics that allows both discrete and continuous random variables.
We restrict the types of possible numerical constraints and, to perform exact inference, we convert the program containing both discrete probabilistic facts and continuous random variables into a program containing only discrete probabilistic facts, similarly to~\citep{DBLP:conf/ilp/GutmannJR10}.
We leverage two existing tools for exact inference, one based on projected answer set enumeration and one based on knowledge compilation.
We also consider approximate inference by sampling either the discretized or the original program.
We tested the four algorithms on different datasets.
The results show that the exact algorithm based on projected answer set enumeration is feasible only for small instances while the one based on knowledge compilation can scale to larger programs.
Approximate algorithms can handle larger instances and sampling the discretized program is often faster than sampling the original program.
However, this has a cost in terms of required memory, since the discretization process adds a consistent number of rules and probabilistic facts.
In the future, we plan to extend our framework to also consider comparisons involving more than one continuous random variable and general expressions, as well as considering lifted inference approaches~\citep{AzzRig23-IJAR-IJ} \revisione{and handle the inference problem with approximate answer set counting~\citep{Kabir2022ApproxASPA}}.

\section*{Acknowledgements} 
This work has been partially supported by Spoke 1 ``FutureHPC \& BigData'' of the Italian Research Center on High-Performance Computing, Big Data and Quantum Computing (ICSC) funded by MUR Missione 4 - Next Generation EU (NGEU) and by Partenariato Esteso PE00000013 - ``FAIR - Future Artificial Intelligence Research'' - Spoke 8 ``Pervasive AI'', funded by MUR through PNRR - M4C2 - Investimento 1.3 (Decreto Direttoriale MUR n. 341 of 15th March 2022) under the Next Generation EU (NGEU).
Both authors are members of the Gruppo Nazionale Calcolo Scientifico -- Istituto Nazionale di Alta Matematica (GNCS-INdAM).
We acknowledge the CINECA award under the ISCRA initiative, for the availability of high-performance computing resources and support.

\subsection*{Competing interests}
The authors declare none.

\bibliographystyle{apalike}
\bibliography{2022hybrid_pasp_TPLP_monofile}

\begin{thebibliography}{}

\bibitem[Alviano and Faber, 2018]{alviano2018aggregates}
Alviano, M. and Faber, W. (2018).
\newblock Aggregates in answer set programming.
\newblock {\em KI-K{\"u}nstliche Intelligenz}, 32(2):119--124.

\bibitem[Azzolini et~al., 2022]{AzzBellRig2022PASTA}
Azzolini, D., Bellodi, E., and Riguzzi, F. (2022).
\newblock Statistical statements in probabilistic logic programming.
\newblock In Gottlob, G., Inclezan, D., and Maratea, M., editors, {\em Logic Programming and Nonmonotonic Reasoning}, pages 43--55, Cham. Springer International Publishing.

\bibitem[Azzolini and Riguzzi, 2023a]{AzzRig2023-AIXIA-IC}
Azzolini, D. and Riguzzi, F. (2023a).
\newblock Inference in probabilistic answer set programming under the credal semantics.
\newblock In Basili, R., Lembo, D., Limongelli, C., and Orlandini, A., editors, {\em {AIxIA} 2023 - Advances in Artificial Intelligence}, volume 14318 of {\em Lecture Notes in Artificial Intelligence}, pages 367--380, Heidelberg, Germany. Springer.

\bibitem[Azzolini and Riguzzi, 2023b]{AzzRig23-IJAR-IJ}
Azzolini, D. and Riguzzi, F. (2023b).
\newblock Lifted inference for statistical statements in probabilistic answer set programming.
\newblock {\em International Journal of Approximate Reasoning}, 163:109040.

\bibitem[Azzolini et~al., 2021]{AzzRigLam21-AIJ-IJ}
Azzolini, D., Riguzzi, F., and Lamma, E. (2021).
\newblock A semantics for hybrid probabilistic logic programs with function symbols.
\newblock {\em Artificial Intelligence}, 294:103452.

\bibitem[Baral et~al., 2009]{DBLP:journals/tplp/BaralGR09}
Baral, C., Gelfond, M., and Rushton, N. (2009).
\newblock Probabilistic reasoning with answer sets.
\newblock {\em Theory and Practice of Logic Programming}, 9(1):57--144.

\bibitem[Bingham et~al., 2018]{bingham2018pyro}
Bingham, E., Chen, J.~P., Jankowiak, M., Obermeyer, F., Pradhan, N., Karaletsos, T., Singh, R., Szerlip, P., Horsfall, P., and Goodman, N.~D. (2018).
\newblock {Pyro: Deep Universal Probabilistic Programming}.
\newblock {\em Journal of Machine Learning Research}.

\bibitem[Brewka et~al., 2011]{brewka2011asp}
Brewka, G., Eiter, T., and Truszczy\'{n}ski, M. (2011).
\newblock Answer set programming at a glance.
\newblock {\em Communications of the ACM}, 54(12):92--103.

\bibitem[Cozman and Mau{\'{a}}, 2016]{DBLP:conf/ilp/CozmanM16}
Cozman, F.~G. and Mau{\'{a}}, D.~D. (2016).
\newblock The structure and complexity of credal semantics.
\newblock In Hommersom, A. and Abdallah, S.~A., editors, {\em PLP 2016}, volume 1661 of {\em {CEUR} Workshop Proceedings}, pages 3--14. CEUR-WS.org.

\bibitem[Cozman and Mau{\'a}, 2017]{cozman2017semantics}
Cozman, F.~G. and Mau{\'a}, D.~D. (2017).
\newblock On the semantics and complexity of probabilistic logic programs.
\newblock {\em Journal of Artificial Intelligence Research}, 60:221--262.

\bibitem[Cozman and Mau{\'{a}}, 2020]{cozman2020pasp}
Cozman, F.~G. and Mau{\'{a}}, D.~D. (2020).
\newblock The joy of probabilistic answer set programming: Semantics, complexity, expressivity, inference.
\newblock {\em International Journal of Approximate Reasoning}, 125:218--239.

\bibitem[Darwiche, 2004]{DBLP:conf/ecai/Darwiche04}
Darwiche, A. (2004).
\newblock New advances in compiling {CNF} into decomposable negation normal form.
\newblock In de~M{\'{a}}ntaras, R.~L. and Saitta, L., editors, {\em 16th European Conference on Artificial Intelligence (ECAI 2004)}, pages 328--332. {IOS} Press.

\bibitem[Darwiche and Marquis, 2002]{DBLP:journals/jair/DarwicheM02}
Darwiche, A. and Marquis, P. (2002).
\newblock A knowledge compilation map.
\newblock {\em Journal of Artificial Intelligence Research}, 17:229--264.

\bibitem[De~Raedt et~al., 2008]{DeR-NIPS08}
De~Raedt, L., Demoen, B., Fierens, D., Gutmann, B., Janssens, G., Kimmig, A., Landwehr, N., Mantadelis, T., Meert, W., Rocha, R., Costa, V., Thon, I., and Vennekens, J. (2008).
\newblock Towards digesting the alphabet-soup of statistical relational learning.
\newblock In {\em {NIPS 2008} Workshop on Probabilistic Programming}.

\bibitem[{De Raedt} et~al., 2007]{DBLP:conf/ijcai/RaedtKT07}
{De Raedt}, L., Kimmig, A., and Toivonen, H. (2007).
\newblock {ProbLog}: A probabilistic {Prolog} and its application in link discovery.
\newblock In Veloso, M.~M., editor, {\em IJCAI 2007}, volume~7, pages 2462--2467. AAAI Press.

\bibitem[Eiter et~al., 2021]{DBLP:conf/kr/EiterHK21}
Eiter, T., Hecher, M., and Kiesel, R. (2021).
\newblock Treewidth-aware cycle breaking for algebraic answer set counting.
\newblock In Bienvenu, M., Lakemeyer, G., and Erdem, E., editors, {\em Proceedings of the 18th International Conference on Principles of Knowledge Representation and Reasoning, {KR} 2021}, pages 269--279.

\bibitem[Feller, 1968]{feller1968introduction}
Feller, W. (1968).
\newblock {\em An Introduction to Probability Theory and Its Applications: Volume I}.
\newblock Wiley series in probability and mathematical statistics. John Wiley \& sons.

\bibitem[Gebser et~al., 2009]{gebser2009projective}
Gebser, M., Kaufmann, B., and Schaub, T. (2009).
\newblock Solution enumeration for projected boolean search problems.
\newblock In van Hoeve, W.-J. and Hooker, J.~N., editors, {\em Integration of AI and OR Techniques in Constraint Programming for Combinatorial Optimization Problems}, pages 71--86, Berlin, Heidelberg. Springer Berlin Heidelberg.

\bibitem[Gehr et~al., 2016]{DBLP:conf/cav/GehrMV16}
Gehr, T., Misailovic, S., and Vechev, M. (2016).
\newblock {PSI:} exact symbolic inference for probabilistic programs.
\newblock In Chaudhuri, S. and Farzan, A., editors, {\em 28th International Conference on Computer Aided Verification ({CAV} 2016), Part {I}}, volume 9779 of {\em LNCS}, pages 62--83. Springer.

\bibitem[Gelfond and Lifschitz, 1988]{gelfond1988stable}
Gelfond, M. and Lifschitz, V. (1988).
\newblock The stable model semantics for logic programming.
\newblock In {\em 5th International Conference and Symposium on Logic Programming (ICLP/SLP 1988)}, volume~88, pages 1070--1080, USA. MIT Press.

\bibitem[Gill, 1977]{doi:10.1137/0206049}
Gill, J. (1977).
\newblock Computational complexity of probabilistic turing machines.
\newblock {\em {SIAM} Journal on Computing}, 6(4):675--695.

\bibitem[Gutmann et~al., 2011a]{DBLP:conf/ilp/GutmannJR10}
Gutmann, B., Jaeger, M., and {De Raedt}, L. (2011a).
\newblock Extending problog with continuous distributions.
\newblock In Frasconi, P. and Lisi, F.~A., editors, {\em ILP 2010}, volume 6489 of {\em LNCS}, pages 76--91. Springer.

\bibitem[Gutmann et~al., 2011b]{DBLP:journals/tplp/GutmannTKBR11}
Gutmann, B., Thon, I., Kimmig, A., Bruynooghe, M., and De~Raedt, L. (2011b).
\newblock The magic of logical inference in probabilistic programming.
\newblock {\em Theory and Practice of Logic Programming}, 11(4-5):663--680.

\bibitem[Hahn et~al., 2022]{plingo2022}
Hahn, S., Janhunen, T., Kaminski, R., Romero, J., R{\"{u}}hling, N., and Schaub, T. (2022).
\newblock plingo: A system for probabilistic reasoning in clingo based on {LPMLN}.

\bibitem[Islam et~al., 2012]{TLP:8688161}
Islam, M.~A., Ramakrishnan, C., and Ramakrishnan, I. (2012).
\newblock Inference in probabilistic logic programs with continuous random variables.
\newblock {\em Theory and Practice of Logic Programming}, 12:505--523.

\bibitem[Jaffar and Maher, 1994]{DBLP:journals/jlp/JaffarM94}
Jaffar, J. and Maher, M.~J. (1994).
\newblock Constraint logic programming: {A} survey.
\newblock {\em Journal of Logic Programming}, 19/20:503--581.

\bibitem[Janhunen et~al., 2017]{janhunen2017clingolinear}
Janhunen, T., Kaminski, R., Ostrowski, M., Schellhorn, S., Wanko, P., and Schaub, T. (2017).
\newblock Clingo goes linear constraints over reals and integers.
\newblock {\em Theory and Practice of Logic Programming}, 17(5-6):872--888.

\bibitem[Kabir et~al., 2022]{Kabir2022ApproxASPA}
Kabir, M., Everardo, F., Shukla, A., Hecher, M., Fichte, J.~K., and Meel, K.~S. (2022).
\newblock {ApproxASP} - a scalable approximate answer set counter.
\newblock In {\em AAAI Conference on Artificial Intelligence}.

\bibitem[Kiesel et~al., 2022]{DBLP:journals/tplp/KieselTK22}
Kiesel, R., Totis, P., and Kimmig, A. (2022).
\newblock Efficient knowledge compilation beyond weighted model counting.
\newblock {\em Theory and Practice of Logic Programming}, 22(4):505--522.

\bibitem[Kimmig et~al., 2011]{DBLP:journals/tplp/KimmigDRCR11}
Kimmig, A., Demoen, B., {De Raedt}, L., Costa, V.~S., and Rocha, R. (2011).
\newblock On the implementation of the probabilistic logic programming language {ProbLog}.
\newblock {\em Theory and Practice of Logic Programming}, 11(2-3):235--262.

\bibitem[Kimmig et~al., 2017]{10.1016/j.jal.2016.11.031}
Kimmig, A., Van~den Broeck, G., and De~Raedt, L. (2017).
\newblock Algebraic model counting.
\newblock {\em J. of Applied Logic}, 22(C):46--62.

\bibitem[Lee and Wang, 2016]{DBLP:conf/kr/LeeW16}
Lee, J. and Wang, Y. (2016).
\newblock Weighted rules under the stable model semantics.
\newblock In Baral, C., Delgrande, J.~P., and Wolter, F., editors, {\em Proceedings of the Fifteenth International Conference on Principles of Knowledge Representation and Reasoning}, pages 145--154. {AAAI} Press.

\bibitem[Lee and Yang, 2017]{DBLP:conf/aaai/LeeY17}
Lee, J. and Yang, Z. (2017).
\newblock {LPMLN}, weak constraints, and {P-log}.
\newblock In Singh, S. and Markovitch, S., editors, {\em Proceedings of the Thirty-First {AAAI} Conference on Artificial Intelligence, February 4-9, 2017, San Francisco, California, {USA}}, pages 1170--1177. {AAAI} Press.

\bibitem[Mau{\'{a}} and Cozman, 2020]{cozman2020complexity}
Mau{\'{a}}, D.~D. and Cozman, F.~G. (2020).
\newblock Complexity results for probabilistic answer set programming.
\newblock {\em International Journal of Approximate Reasoning}, 118:133--154.

\bibitem[Michels et~al., 2015]{DBLP:journals/ai/MichelsHLV15}
Michels, S., Hommersom, A., Lucas, P. J.~F., and Velikova, M. (2015).
\newblock A new probabilistic constraint logic programming language based on a generalised distribution semantics.
\newblock {\em Artificial Intelligence}, 228:1--44.

\bibitem[Mitzenmacher and Upfal, 2017]{mitzenmacher2017probability}
Mitzenmacher, M. and Upfal, E. (2017).
\newblock {\em Probability and computing: Randomization and probabilistic techniques in algorithms and data analysis}.
\newblock Cambridge university press.

\bibitem[Nickles, 2018]{DBLP:conf/ilp/Nickles18a}
Nickles, M. (2018).
\newblock Differentiable {SAT/ASP}.
\newblock In Bellodi, E. and Schrijvers, T., editors, {\em Proceedings of the 5th International Workshop on Probabilistic Logic Programming, {PLP} 2018, co-located with the 28th International Conference on Inductive Logic Programming ({ILP} 2018), Ferrara, Italy, September 1, 2018}, volume 2219 of {\em {CEUR} Workshop Proceedings}, pages 62--74. CEUR-WS.org.

\bibitem[Nickles and Mileo, 2015]{DBLP:conf/iclp/NicklesM15}
Nickles, M. and Mileo, A. (2015).
\newblock A hybrid approach to inference in probabilistic non-monotonic logic programming.
\newblock In Riguzzi, F. and Vennekens, J., editors, {\em Proceedings of the 2nd International Workshop on Probabilistic Logic Programming co-located with 31st International Conference on Logic Programming ({ICLP} 2015), Cork, Ireland, August 31st, 2015}, volume 1413 of {\em {CEUR} Workshop Proceedings}, pages 57--68. CEUR-WS.org.

\bibitem[Pfeffer, 2016]{pfeiffer2016}
Pfeffer, A. (2016).
\newblock {\em Practical Probabilistic Programming}.
\newblock Manning Publications.

\bibitem[Phan et~al., 2019]{phan2019composable}
Phan, D., Pradhan, N., and Jankowiak, M. (2019).
\newblock Composable effects for flexible and accelerated probabilistic programming in {NumPyro}.

\bibitem[Riguzzi, 2013]{Rig13-FI-IJ}
Riguzzi, F. (2013).
\newblock {MCINTYRE}: A {Monte Carlo} system for probabilistic logic programming.
\newblock {\em Fundamenta Informaticae}, 124(4):521--541.

\bibitem[Riguzzi, 2022]{Rig23-BKaddress}
Riguzzi, F. (2022).
\newblock {\em Foundations of Probabilistic Logic Programming Languages, Semantics, Inference and Learning, Second Edition}.
\newblock River Publishers, Gistrup, Denmark.

\bibitem[Rocha and Gagliardi~Cozman, 2022]{DBLP:conf/kr/RochaC22}
Rocha, V. H.~N. and Gagliardi~Cozman, F. (2022).
\newblock A credal least undefined stable semantics for probabilistic logic programs and probabilistic argumentation.
\newblock In Kern{-}Isberner, G., Lakemeyer, G., and Meyer, T., editors, {\em Proceedings of the 19th International Conference on Principles of Knowledge Representation and Reasoning, {KR} 2022}, pages 309--319.

\bibitem[Sato, 1995]{DBLP:conf/iclp/Sato95}
Sato, T. (1995).
\newblock A statistical learning method for logic programs with distribution semantics.
\newblock In Sterling, L., editor, {\em ICLP 1995}, pages 715--729. MIT Press.

\bibitem[Shterionov et~al., 2015]{DBLP:conf/ilp/ShterionovRVKMJ14}
Shterionov, D.~S., Renkens, J., Vlasselaer, J., Kimmig, A., Meert, W., and Janssens, G. (2015).
\newblock The most probable explanation for probabilistic logic programs with annotated disjunctions.
\newblock In Davis, J. and Ramon, J., editors, {\em ILP 2014}, volume 9046 of {\em LNCS}, pages 139--153, Berlin, Heidelberg. Springer.

\bibitem[Totis et~al., 2023]{totis_de_raedt_kimmig_2023}
Totis, P., De~Raedt, L., and Kimmig, A. (2023).
\newblock {smProbLog}: Stable model semantics in problog for probabilistic argumentation.
\newblock {\em Theory and Practice of Logic Programming}, pages 1--50.

\bibitem[Tran et~al., 2017]{DBLP:journals/corr/TranHSBMB17}
Tran, D., Hoffman, M.~D., Saurous, R.~A., Brevdo, E., Murphy, K., and Blei, D.~M. (2017).
\newblock Deep probabilistic programming.
\newblock {\em CoRR}, abs/1701.03757.

\bibitem[Tran et~al., 2016]{tran2016edward}
Tran, D., Kucukelbir, A., Dieng, A.~B., Rudolph, M., Liang, D., and Blei, D.~M. (2016).
\newblock {Edward: A library for probabilistic modeling, inference, and criticism}.
\newblock {\em arXiv preprint arXiv:1610.09787}.

\bibitem[Tuckey et~al., 2021]{tuckey2021PASOCS}
Tuckey, D., Russo, A., and Broda, K. (2021).
\newblock {PASOCS}: A parallel approximate solver for probabilistic logic programs under the credal semantics.
\newblock {\em arXiv}, abs/2105.10908.

\bibitem[van~de Meent et~al., 2021]{vandemeent2021introduction}
van~de Meent, J.-W., Paige, B., Yang, H., and Wood, F. (2021).
\newblock An introduction to probabilistic programming.

\bibitem[Van~den Broeck et~al., 2010]{DBLP:conf/aaai/BroeckTOR10}
Van~den Broeck, G., Thon, I., van Otterlo, M., and De~Raedt, L. (2010).
\newblock {DTProbLog}: A decision-theoretic probabilistic {Prolog}.
\newblock In Fox, M. and Poole, D., editors, {\em Proceedings of the Twenty-Fourth AAAI Conference on Artificial Intelligence}, pages 1217--1222. AAAI Press.

\bibitem[Virtanen et~al., 2020]{2020scipy}
Virtanen, P., Gommers, R., Oliphant, T.~E., Haberland, M., Reddy, T., Cournapeau, D., Burovski, E., Peterson, P., Weckesser, W., Bright, J., {van der Walt}, S.~J., Brett, M., Wilson, J., Millman, K.~J., Mayorov, N., Nelson, A. R.~J., Jones, E., Kern, R., Larson, E., Carey, C.~J., Polat, {\.I}., Feng, Y., Moore, E.~W., {VanderPlas}, J., Laxalde, D., Perktold, J., Cimrman, R., Henriksen, I., Quintero, E.~A., Harris, C.~R., Archibald, A.~M., Ribeiro, A.~H., Pedregosa, F., {van Mulbregt}, P., and {SciPy 1.0 Contributors} (2020).
\newblock {{SciPy} 1.0: Fundamental Algorithms for Scientific Computing in Python}.
\newblock {\em Nature Methods}, 17:261--272.

\end{thebibliography}

\end{document}